\documentclass[twoside,11pt]{article}

\usepackage{blindtext}
\usepackage{graphicx}
\usepackage{amsmath}
\usepackage{amssymb}
\usepackage{bbm}
\usepackage{pgfplots}
\pgfplotsset{compat=1.18}
\usepackage{multirow}
\usepackage{float}
\usepackage{mathtools}

\usepackage[figuresleft]{rotating}
\usepackage{booktabs}
\usepackage{arydshln}
\usepackage[caption = false]{subfig}
\usepackage{jmlr2e}

\DeclareMathOperator*{\argmax}{arg\,max}
\DeclareMathOperator*{\argmin}{arg\,min}

\usepackage{lastpage}
\jmlrheading{25}{2024}{(Submitted)}{5/30/24; Revised --}{--}{21-0000}{Bertsimas and Peroni}

\ShortHeadings{Interpretable and Adaptive Model Selection}{Bertsimas and Peroni}
\firstpageno{1}

\begin{document}

\title{Policy Trees for Prediction: Interpretable and Adaptive Model Selection for Machine Learning}

\author{\name Dimitris Bertsimas \email dbertsim@mit.edu \\
       \addr Sloan School of Management and Operations Research Center\\
       Massachusetts Institute of Technology\\
       Cambridge, MA 02139, USA
       \AND
       \name Matthew Peroni \email mperoni1@mit.edu \\
       \addr Operations Research Center\\
       Massachusetts Institute of Technology\\
       Cambridge, MA 02139, USA}

\editor{--}

\maketitle

\begin{abstract}
    As a multitude of capable machine learning (ML) models become widely available in forms such as open-source software and public APIs, central questions remain regarding their use in real-world applications, especially in high-stakes decision-making. Is there always one best model that should be used? When are the models likely to be error-prone? Should a black-box or interpretable model be used? In this work, we develop a prescriptive methodology to address these key questions, introducing a tree-based approach, \textbf{Optimal Predictive-Policy Trees (OP\textsuperscript{2}T)}, that yields interpretable policies for adaptively selecting a predictive model or ensemble, along with a parameterized option to reject making a prediction. We base our methods on learning globally optimized prescriptive trees. Our approach enables interpretable and adaptive model selection and rejection while only assuming access to model outputs. By learning policies over different feature spaces, including the model outputs, our approach works with both structured and unstructured datasets. We evaluate our approach on real-world datasets, including regression and classification tasks with both structured and unstructured data. We demonstrate that our approach provides both strong performance against baseline methods while yielding insights that help answer critical questions about which models to use, and when.
\end{abstract}


\section{Introduction}
As increasingly advanced machine learning (ML) algorithms and pre-trained models become democratized and readily available across hundreds of open-source software packages and APIs, such as Scikit-learn \citep{scikit-learn}, PyTorch \citep{paszke2019pytorch}, and HuggingFace \citep{wolf2020huggingfaces}, central questions remain regarding their use in high-stakes decision-making. 
\newpage
\begin{itemize}
    \item Given a collection of models, which should we use for our application, and when?
    \item Should we use a model ensemble?
    \item When are the models likely to be error-prone?
    \item  Should we use a black-box model or an interpretable model?
\end{itemize}

These questions are, to some degree, open-ended and philosophical. Considerations may include computational resources, inference time, and the qualitative importance of interpretability. To make them more concrete, we aim to address these questions with respect to model performance. For example, when answering whether a model ensemble should be used, we are only concerned with whether the ensemble improves out-of-sample performance. Similarly, we address the question of when models may be error-prone by identifying partitions of the feature space where the models demonstrate performance below some acceptable threshold. While there is extensive literature on the topics of model selection \citep{ding2018model}, ensembling \citep{dong_survey_2020}, and rejection learning \citep{hendrickx2021machine}, there lacks a comprehensive, holistic modeling framework that can provide clear, interpretable answers to the questions we have posed. 

In this work, we develop a prescriptive methodology for adaptively selecting the best model (or ensemble) from a collection of models to make a prediction for a given input or to reject predicting entirely if it is likely that all models will perform poorly. Our approach, which we refer to as Optimal Predictive-Policy Trees (OP\textsuperscript{2}Ts), is a tree-based method that partitions the feature space based on the relative performance of each constituent model. By constructing an interpretable policy model, we uncover sub-populations of the data where different models perform the best, and consequently, sub-populations where some (or all) models are likely to fail. With a simple extension, we are also able to consider a fixed set of model ensembles with different weights, providing an interpretable and adaptive scheme for model ensembling. Within this framework, we also learn optimal policies for rejection, parameterized by the tolerance for error in predictions. In this work, by \textbf{learning with rejection}, we are referring to the class of learning algorithms that, in addition to making predictions, have the option of not making a prediction. This includes algorithms that learn to reject based on the scores of an existing model \citep{chow_optimum_1970, bartlett_classification_2008}, as well as approaches that learn a predictor and rejector jointly \citep{ortner_learning_2016}. We show empirically on real-world datasets that our approach consistently yields better performance than any of the constituent models individually, including model ensembles.

Our work is motivated by overlapping gaps in the Mixture of Experts (ME) and Rejection Learning literature. The first is that model ensemble and ME approaches are generally not interpretable, with most recent developments focusing on deep neural network (DNN) models and assignment modules \citep{shazeer2017outrageously, riquelme2021scaling, masoudnia_mixture_2014}. They cannot tell the practitioner \textit{why} a certain set of weights was assigned to each of the constituent models, or why a certain model was selected. Recently, a framework referred to as the Interpretable Mixture of Experts (IME) was introduced in part to address this gap \citep{ismail_interpretable_2023}. However, the IME approach requires that all constituent models be differentiable (e.g. logistic regression, soft decision trees, DNNs) and that all models have accessible weights that can be tuned. In contrast, this work is motivated by the observation that in many real-world applications, practitioners already have a collection of models under consideration, some of which may not be differentiable. Further, there may be models that are differentiable but for privacy, cost, or licensing reasons, cannot be further trained. In these situations, such ME approaches \textbf{do not integrate with existing models}, and therefore become their own, separate modeling approach. The rise of publicly available APIs for large, multi-modal models (e.g. ChatGPT, PaLM) suggests that there will be an increasing desire to incorporate models that either cannot practically, or in some cases legally, be fine-tuned. 

As our methodology generates tree-based policies, it is also related to the hierarchical ME literature. This line of research originated with \cite{jordan_hierarchical_1994}, which introduces the notion of a soft decision tree that partitions the feature space using gating functions and weights the outputs of a collection of models that are trained jointly with the soft tree. While there have been a variety of papers that have built on this original work, such as \cite{solomatine_semi-optimal_2004} and \cite{irsoy_soft_2012}, the high-level concept of using soft splits and training a collection of experts jointly has not changed. However, existing hierarchical ME models rely on soft splits, which reduces their interpretability, especially if non-linear gating functions are used. This line of work also has not addressed the incorporation of learning with rejection.

Similar to the ME literature, there is also the assumption in recent rejection learning literature that the predictor can be optimized jointly with the rejector \citep{ortner_learning_2016, charoenphakdee_classification_2021}. Earlier literature that focused on learning a rejector for a pre-existing predictor, using the predictor's scores to create a performance-rejection trade-off curve, was either primarily concerned with learning a single predictor-rejector pair \citep{chow_optimum_1970, bartlett_classification_2008}, or considered only binary classification and did not learn a policy over the feature space \citep{provost_robust_2001}. In this work, we incorporate rejection learning with many constituent models, assuming only query access to these models. Further, since our approach yields interpretable policies, we can provide faithful descriptions of the contexts in which it is better to reject, and not use any of the available models. As we discuss in Section \ref{sec:results}, such interpretability can be practically useful when creating real-world policies for high-stakes decision-making with machine learning systems.

 As a motivating example, we point to recent work by \cite{boussioux_hurricane_2022}, which develops a multi-modal machine learning approach to hurricane forecasting. In their work, \cite{boussioux_hurricane_2022} develop a collection of models from different model classes. Specifically, they incorporate tree-based, neural network, and physics-based models, and demonstrate that combining the predictions of these models yields more accurate and robust predictions than traditional physics-based modeling alone. It is clear in this case that only a subset of the constituent models can be fine-tuned, and they are not all differentiable. Further, their work demonstrates the benefits of including constituent models from a variety of model classes, rather than enforcing that they all come from one class. In a similar vein, \cite{soenksen_integrated_2022} demonstrates the dramatic impact of multi-modality for prediction in the healthcare domain, relying on constituent models from a variety of model classes to handle each modality, and additionally shows that each modality has varying importance depending on the task. For any work similar to this, it is vital to have a model selection and rejection approach that relies solely on model outputs and does not assume the models can be further trained.

Motivated by these observations, the main contributions of our work are as follows:
\begin{enumerate}
    \item We develop a methodology to create interpretable policies for adaptive model selection and rejection. We do so while assuming only query access to the constituent models. Further, we develop approaches that work for both structured and unstructured data.
    \item We develop our method for both regression and classification tasks, and demonstrate empirically that our method is comparable with or improves over baseline approaches while yielding interpretable policies. We demonstrate how one can use the resulting policies to perform sub-group analysis, finding subsets of the data on which different models perform well.
    \item We incorporate parameterized rejection, which, within our interpretable framework, can identify settings where all available models are likely to be error-prone. We further explore and demonstrate how adding a rejection option can serve as a form of regularization that empirically yields additional tree stability.
    \item We offer theoretical insights on the conditions under which the proposed approach can learn a policy that is stronger than the best single model in hindsight, as well as alternative approaches that do not consider the relative rewards of each model.
\end{enumerate}

We begin by introducing the class of prescriptive decision trees we will use to learn our model selection policies and the global optimization approach we take to learn trees in this class in Section \ref{sec:opt}. In Section \ref{sec:oms-clf}, we develop our approach for adaptive model selection in the context of classification, before introducing the addition of a parameterized rejection option in Section \ref{sec:oms-clf-rej}. We then introduce the corresponding setup for regression in Section \ref{sec:oms-reg}. Following the development of our methods, we provide some theoretical justification for our approach in Section \ref{sec:theory}. Finally, in Section \ref{sec:results} we evaluate our methodology on a variety of real-world datasets and benchmark our approach against related, predictive techniques. We provide additional experiments and proofs in the Appendices.

\section{Optimal Policy Trees}\label{sec:opt}
At the center of our methodology is learning interpretable, tree-based policies that are capable of capturing nonlinear interactions and partitioning the feature space. To accomplish this, we base our approach on the Optimal Policy Tree (OPT) formulation, introduced in \cite{amram_optimal_2022}. The setup is as follows. Suppose we have some dataset $\{x_i\}_{i=1}^n$ of size $n$, where $x_i \in \mathbb{R}^d$ along with $m$ possible treatments $T = \{t_1, \dots, t_m\}$. Suppose we are given a reward for each observation $i$ and treatment $t$, denoted by $R_{it}$. In many cases, such as observational data in medicine, the counterfactuals, the reward for sample $i$ under each treatment, may not be known. In these settings, an additional reward estimation step is required. However, as we will demonstrate in Section \ref{sec:oms-clf}, in our approach, no reward estimation will be necessary. We would like to learn a function $f: \mathbb{R}^d \rightarrow T$ that prescribes a treatment for any input sample. Specifically, we will learn this function from the class of decision trees. Our objective is to maximize (or minimize) the total reward over all samples using the treatments prescribed by the decision tree. Assuming we are interested in maximizing the reward, our objective becomes
\begin{equation}
 \max_{\tau(\cdot)} \sum_{i=1}^n\sum_{t=1}^m \mathbbm{1}\{\tau(x_i) = t\}R_{it},
\end{equation}\label{eq:opt-base}
where we maximize over the class of decision trees $\tau$ up to some depth $D$ and $\tau(x_i)$ is the prescription made by the decision tree for sample $x_i$. To avoid overfitting, we add a penalty to the objective for tree complexity, measured by the number of splits in the tree $\tau$. The formulation then becomes
\begin{equation}
    \max_{\tau(\cdot)} \sum_{i=1}^n\sum_{t=1}^m \mathbbm{1}\{\tau(x_i) = t\}R_{it} + \lambda \cdot \text{numsplits}(\tau),
\end{equation}
where $\lambda \in \mathbb{R}_{+}$ is a parameter to control the complexity penalty. Focusing on the prescription in the leaf nodes, we introduce the leaf assignment function $v: \mathbb{R}^d \rightarrow L$ for a given tree $\tau$ with the set of leaf nodes $L$, where $|L| = q$. Denoting the prescriptions for each leaf by $\mathbf{z} = (z_1, \dots, z_q)$, we can rewrite our formulation as
\begin{equation}
        \max_{v, \mathbf{z}} \sum_{i=1}^n\sum_{l=1}^q \mathbbm{1}\{v(x_i) = l\}R_{iz_l} + \lambda \cdot \text{numsplits}(v).
\end{equation}
Notice that this problem is separable in the leaves, such that we can rewrite the formulation as
\begin{equation}
        \max_{v, \mathbf{z}} \sum_{l=1}^q \sum_{i : v(x_i) = l} R_{iz_l} + \lambda \cdot \text{numsplits}(v).    
\end{equation}
Then, for any given tree structure $\tau$ and corresponding leaf assignment function $v$, the optimal prescription in each leaf is given by
\begin{equation}
    z_l = \argmin_{t \in T} \sum_{i : v(x_i) = l} R_{it},
\end{equation}
which can be solved by enumerating the possible treatments. We add constraints for the maximum depth, $D$, of the tree and the minimum number of samples for each leaf, $c_{min}$. Together, the parameters $\lambda, c_{min}$ and $D$ help control overfitting. Putting this together, our final formulation can be written as
\begin{equation}
\begin{split}
    \max_{v, \mathbf{z}} & \quad \sum_{l=1}^q \sum_{i : v(x_i) = l} R_{iz_l} + \lambda \cdot \text{numsplits}(v) \\
    \text{s.t.} & \quad \text{depth}(v) \leq D,\\
         & \quad \sum_{i=1}^n \mathbbm{1}\{v(x_i) = l\} \geq c_{min} \quad \forall l.
\end{split}
\end{equation}
The tree $\tau$ is then optimized using the Optimal Trees framework given in Section 8.4 of \cite{bertsimas_machine_2019}. Specifically, rather than using a greedy heuristic, we use a global optimization approach, based on coordinate descent, to optimize the tree structure and the prescriptions jointly. The hyperparameters $D$ and $c_{min}$ are optimized using a traditional grid search over a discrete set of values, while $\lambda$ is optimized through a pruning approach that generates a sequence of trees and identifies the value of $\lambda$ that minimizes the validation loss. The problem posed in Eq. \eqref{eq:opt-base} was also formulated by \cite{zhou_offline_2018} and \cite{biggs_model_2021}, however, they both used greedy heuristics to fit their policy trees. In \cite{amram_optimal_2022}, the authors demonstrate that training policy trees using a global optimization methodology can yield significant performance and interpretability advantages over greedy heuristic approaches.

For this work, we focus on the class of parallel-split decision trees, but our formulation can be easily extended to hyperplane splits, as in \cite{bertsimas_optimal_2017}. We do this to maintain the focus of this work, and preserve interpretability. In addition, we found in preliminary experiments that hyperplane splits did not yield a significant change in performance for the datasets we evaluated. 

\section{Adaptive Model Selection for Classification and Regression}\label{sec:oms-clf}
In this section, we develop our Optimal Predictive-Policy Tree methodology for classification and regression tasks. For both settings, we introduce the option for parameterized rejection and provide an analysis of how this parameterization impacts the resulting policy model. 

\subsection{An Illustrative Example}
We begin by discussing an illustrative example of our OP\textsuperscript{2}T method. This is a hypothetical case that demonstrates what a practical application of our approach could look like. Following our motivating example given by \cite{boussioux_hurricane_2022}, let us consider the problem of forecasting hurricane severity. To simplify, let us assume we are interested in a standard binary classification task, using the input feature space to predict whether a hurricane will be a Category 3 or above upon landfall. Suppose we have access to tabular features $\mathcal{X}_{tab}$ providing the physical conditions of the hurricane, such as wind speed, pressure, and distance to land. In addition, suppose we have access to satellite images of the forming hurricane, which we denote by $\mathcal{X}_{img}$. Using this data and previous work in hurricane forecasting, suppose we create a collection of models, comprised of a logistic regression model, $h_{lr}$, and a boosted tree model, $h_{boost}$, fit on $\mathcal{X}_{tab}$ alone, a convolutional neural network (CNN), $h_{cnn}$, trained on $\mathcal{X}_{img}$, and a physics simulator based on the Navier-Stokes equations (NS), $h_{ns}$, which relies on both $\mathcal{X}_{img}$ and $\mathcal{X}_{tab}$ to specify the state and initial conditions. It is clear that in this setting, further fine-tuning all of the models will not be possible, or even desired.

In this setting, we have four constituent models, $\mathbf{h} = [h_{lr}, h_{boost}, h_{cnn}, h_{ns}]$ to choose from. In addition, we would like to consider ensembles of these four models. In our work, we will consider a finite set of ensemble weights. These ensemble weights are specified a priori. A reasonable example would be to include each individual model (i.e., an ensemble where the weight is placed entirely on one model), a mean ensemble, and pair-wise ensembles (averaging the predictions of a pair of models). Our goal is then to adaptively select the best model, or ensemble, for each input sample.

\begin{figure}
    \centering
    \includegraphics[width=0.8\linewidth]{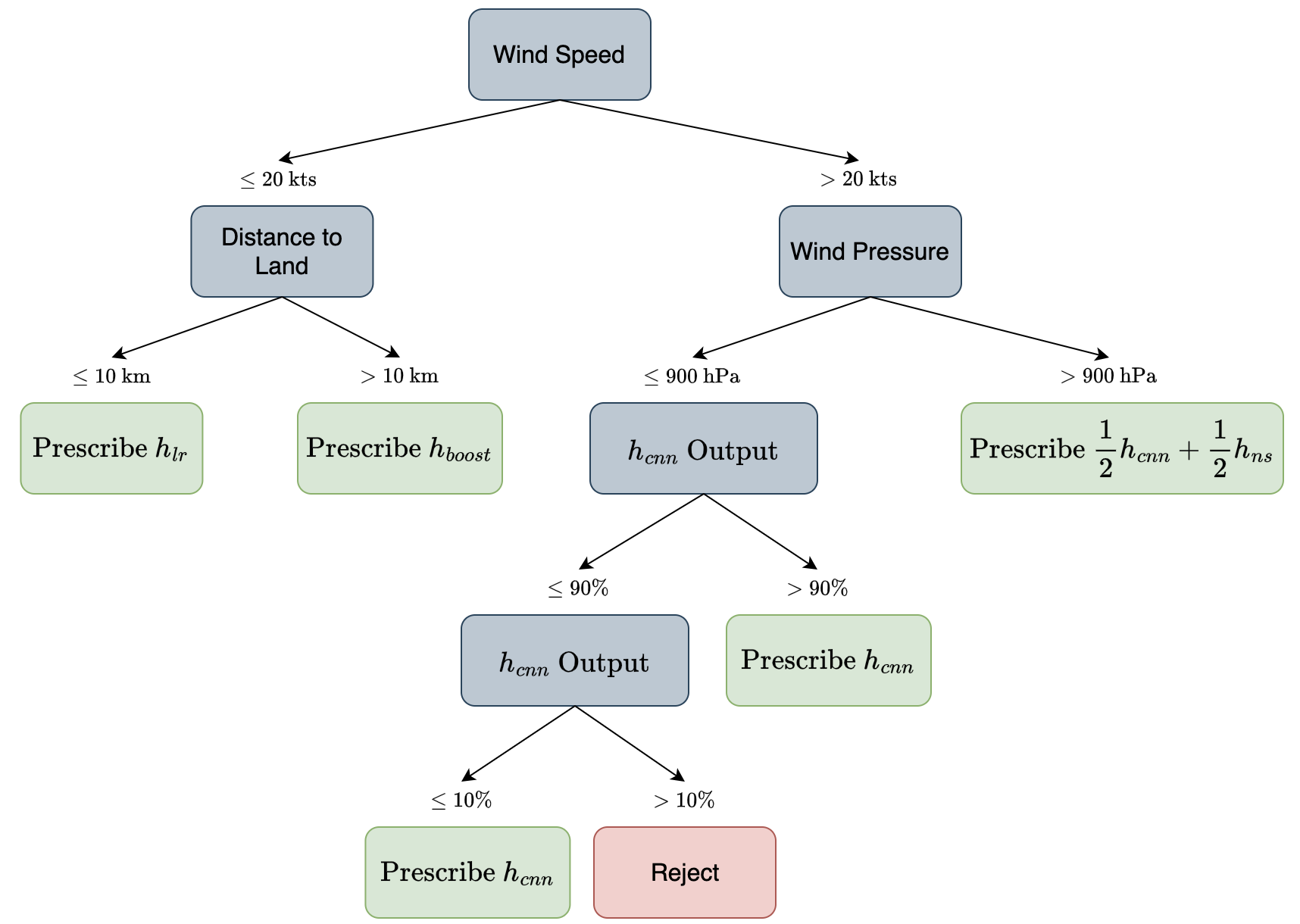}
    \caption{An illustrative example of the OP\textsuperscript{2}T approach to adaptive model selection.}
    \label{fig:hurricane}
\end{figure}

For each model or ensemble, we can compute a reward for using that model on each input sample. How we define this reward function is flexible, and we explore multiple options in the remainder of Section \ref{sec:oms-clf}. In addition, we can add an option to reject making a prediction entirely by creating a dummy model $h_r$ and assigning some reward for selecting this rejection model. With the models, ensemble weights, and rewards fixed, what remains is to specify the feature space for our policy model. For classification tasks, in addition to the original feature space, we may also want to use the model outputs themselves as features, as a measure of model confidence. In this case, we denote the model output space as $\mathcal{X}_{conf} = [0,1]^4$, representing the probability output of each of the four constituent models. In our approach, our policy function would then be an OPT fit on the feature space $\mathcal{Z} = \mathcal{X}_{tab} \cup \mathcal{X}_{conf}$, where the prescriptions are the set of ensemble weights and the dummy rejection model. The resulting policy tree could look like the example given in Figure \ref{fig:hurricane}. We can then use this tree as a routing function to send input samples to predictive models. Further, the resulting tree can provide answers to the motivating questions we posed. For this hypothetical example, we could answer these questions as follows.

\noindent \textbf{Should a model ensemble be used? If so, when?} Yes, when the wind speed is greater than $20 \text{kts}$ and the wind pressure is greater than $900 \text{hPa}$, it is best to use an ensemble of the CNN and the NS simulation.

\noindent \textbf{When are the models likely to be error-prone?} Under the combination of high wind speed and low wind pressure, if the CNN model is not confident, $h_{cnn}(x) \in (0.1, 0.9)$, then it is likely that all the models will be error-prone,  resulting in the tree prescribing rejection.

\noindent \textbf{Should an interpretable model be used?} Yes, when the wind speed is low and the hurricane is close to land, the linear regression model performs best. Otherwise, it is better to use a black-box model.

\subsection{Adaptive Model Selection for Classification} \label{subsec: Class}
In this section, we formalize our approach. We first consider classification tasks, where we have some input data $\{x_i\}_{i=1}^n$ from a feature space $\mathcal{X}$ and label data $\{y_i\}_{i=1}^n \in \mathcal{C}^n$, where $C = \{1,\dots, K\}$. We assume access to a collection of $m$ constituent models $\{h_1, \dots, h_m\}$, $h_i: \mathcal{X} \rightarrow \Delta^{K}$, where $\Delta^K$ is the unit simplex, and all models are fit on some training set. We denote by $\mathbf{h} = [h_1,\dots, h_m]$ the vectorized form of the constituent models such that $\mathbf{h}: \mathcal{X} \rightarrow \mathbb{R}^{m \times K}$. Let $\hat y_{ijk} = h_i(x_j)_k$, such that $\hat y_{ijk}$ is the output probability of class $k$ for sample $x_j$ from model $h_i$. Given a fixed, finite set of weights $\mathbf{W} \subset \Delta^m$ over the unit simplex, we want to learn a function $f: \mathcal{X} \rightarrow \mathbf{W}$ such that for any $x \in \mathcal{X}$, $f$ selects the best set of weights $\mathbf{w} \in \mathbf{W}$ such that the resulting model (or ensemble) $\mathbf{w}^T\mathbf{h}$ makes the best prediction on $x$. In the case that $\mathcal{X}$ is not some real space (e.g. unstructured language data), we instead want to learn our policy model $f$ over some real-valued side-information $\{z_i\}_{i=1}^n \in \mathcal{Z}^n$, where $\mathcal{Z}$ could be a subset of the features space $\mathcal{X}$ or separate data entirely, including the model outputs $\mathbf{h}(x)$ themselves. We can measure the quality of a model's prediction on a given sample, or the reward, in a few ways. For classification, we consider misclassification error and cross-entropy, defined as follows:
\begin{equation}
    R_{CE}(x_j, y_j, h_i) = \sum_{k=1}^K  \mathbbm{1} \{y_{j} = k\}\log(\hat y_{ijk}),
\end{equation}
\begin{equation}
    R_{MIS}(x_j, y_j, h_i) = \mathbbm{1}\{y_j = \argmax_{k}\hat y_{ijk}\}.
\end{equation}
We can equivalently define the rewards for a set of weights $\mathbf{w} \in \mathbf{W}$,
\begin{equation}
    R_{CE}(x_j, y_j, \mathbf{w}, \mathbf{h}) = \sum_{k=1}^K  \mathbbm{1} \{y_{j} = k\}\log((\mathbf{w}^T\mathbf{h}(x_j))_k),
\end{equation}
\begin{equation}
    R_{MIS}(x_j, y_j, \mathbf{w}, \mathbf{h}) = \mathbbm{1}\{y_j = \argmax_{k}(\mathbf{w}^T\mathbf{h}(x_j))_k\}.
\end{equation}
 We introduce both reward functions because they present inherent trade-offs. The cross-entropy reward is more informative than the misclassification reward. Making the reasonable assumption that the model predictions are fairly smooth over the feature space, $R_{CE}$ will also form a smoother reward surface than $R_{MIS}$. One risk with using the cross-entropy reward is that models that are over-confident, as is often the case with neural networks \citep{wang_rethinking_2021}, may be favored despite their ability to separate the data being similar to, or worse than, other models available. Since our policy prescribes models, rather than making class predictions directly, neither reward function requires committing to a threshold for class prediction a-priori. However, using $R_{MIS}$ requires selecting a prediction threshold for generating the rewards, which will impact the resulting tree structure and model prescriptions. As we discuss in Section \ref{sec:oms-clf-rej}, $R_{MIS}$ benefits from a simpler interpretation, as model prescriptions are based directly on minimizing misclassification error.

With these definitions, we can construct our reward matrix $\mathbf{R} \in \mathbb{R}^{n \times m}$ corresponding to the reward for using each model $h_i$ to predict the label for each sample $x_j$. We define our action space as the set of weights $\mathbf{W}$ defined above. Then, prescribing model $h_i$ corresponds to the action $\mathbf{e}_i = [0,\dots, 1, \dots, 0] \in \mathbf{W}$. Throughout this work, we assume $\mathbf{e}_i \in \mathbf{W}$ for all $i \in [m]$. That is, we can always take the action of prescribing an individual constituent model. As defined, any fixed ensemble of constituent models can be included in the action space $\mathbf{W}$. In Section \ref{sec:results}, we demonstrate this flexibility and show empirically that our approach uncovers partitions of the feature space where different ensemble weights are better suited.

We are now able to formulate our objective as a policy-learning problem, where the set of actions $\mathbf{W} = \{\mathbf{w}_1, \dots, \mathbf{w}_q\}$ is the set of weights over the constituent models that we can prescribe, our state space is $\mathcal{X}$ (or $\mathcal{Z}$), and the rewards for prescribing each model is given by our reward matrix $\mathbf{R}$. We propose learning an interpretable and adaptive policy model by using the Optimal Trees algorithm to produce a policy tree $T_{O}$ that learns to prescribe constituent models (or ensembles) given the input data by maximizing the reward $R$ over the class of decision trees. We fit $T_{O}$ using a validation dataset that is separate from the data used to train the constituent models. Throughout this work, we refer to this general approach as Optimal Predictive-Policy Trees (OP\textsuperscript{2}Ts). In the next section, we discuss a simple extension of this framework to incorporate a model rejection option. 

\subsection{Classification with Rejection}\label{sec:oms-clf-rej}
Building on our approach in Section \ref{subsec: Class}, we  introduce a simple extension that allows us to incorporate rejection into this framework. In addition to the $m$ constituent models $h_1,\dots, h_m$, we now introduce a dummy rejection model $h_r$ that we add to the action space of our formulation. We must assign reward values $R(\cdot, h_r)$ so that our policy learning algorithm can compare the action of rejection against the action of prescribing the other models. One approach is to fix some constants $\boldsymbol{\alpha} = (\alpha_1,\dots,\alpha_K) \in [0,1)^K$ and suppose the output of the rejection model is always exactly within $\alpha_k$ of predicting the true label $k$. That is, for some input $x_j$ with label $y_j = k$, we define $\hat y_{rjk} = 1 - \alpha_k$. For example, in the binary classification setting, we define the output of the rejection model as
\begin{equation}
    h_r(x, y) = \mathbbm{1} \{y = 1\} (1 - \alpha_1) + \mathbbm{1}\{y=0\} \alpha_0.
\end{equation}
Therefore, if we were to set $\boldsymbol{\alpha} = (0.3, 0.2)$, then $h_r(x, y=1) = 0.7$ and $h_r(x, y=0) = 0.2$. In our experiments, we set $\boldsymbol{\alpha}$ as a constant vector such that $\alpha_i = \alpha$ for all $i \in [K]$ and generate OP\textsuperscript{2}Ts for a range of values for $\alpha$. In practice, one may wish to select a non-constant $\boldsymbol{\alpha}$ to characterize a model that has varying performance on different classes.
Notice that, if we assume $\alpha_k < \frac{1}{K}$ for all $k \in [K]$, $h_r$ is a consistent classifier. That is, for any sample $(x,y)$, we have $y = \argmax_i h_r(x)_i$. Therefore, to construct rewards for rejection, we assume that $h_r$ is \textbf{perfectly calibrated}, such that $P(y = k | h_r(x)_k = 1 - \alpha_k) = 1 - \alpha_k$. We can then define the reward as
\begin{equation}
    R_{CE}(x_j, y_j, h_r) = \sum_{k=1}^K  \mathbbm{1} \{y_{j} = k\}\log(1-\alpha_k),
\end{equation}
\begin{equation}
    R_{MIS}(x_j, y_j, h_r) = 1-\alpha_{y_j}.
\end{equation}
The misclassification reward for the rejection model can be interpreted as the expected reward. With the rewards defined for the rejection model, we can proceed with our approach as in Section \ref{sec:oms-clf}, with the action space extended to $\mathbf{\bar W} = \mathbf{W} \cup \{h_r\}$. Since our decision tree formulation prescribes the action that maximizes (or minimizes) the total reward in each leaf, using $R_{MIS}$, we can immediately conclude that rejection is prescribed in a leaf if and only if the accuracy of the constituent models is less than the accuracy of the rejection model parameterized by $\boldsymbol{\alpha}$. In contrast, due to the non-linear transformation of the logarithm in the reward function $R_{CE}$, which more severely punishes predictions that are far from the true label, the decision to reject does not have a clear interpretation in terms of a performance metric. In Appendix \ref{proof:class-rej-properties} we provide a result that gives the conditions for an OP\textsuperscript{2}T to prescribe rejection in a leaf under the $R_{CE}$ reward function. Further, as $\alpha_k \rightarrow 0$ for all $k \in [K]$, $T_{O}$ will converge to prescribing rejection always, assuming for each leaf $l$, there is at least one sample $x_i \in l$ and constituent model $h_j$ such that $\hat y_{ijy_i} < 1$. We also formalize this idea in Section \ref{sec:theory-rejection}. Since the reward function $R_{CE}$ is differentiable with respect to the rejection parameters $\boldsymbol{\alpha}$ and model outputs, we can also reason about how our prescriptions would change for a fixed tree. This enables straightforward sensitivity analysis of learned OP\textsuperscript{2}Ts with the $R_{CE}$ reward. In Appendix \ref{apdx:rej-intervals}, we develop a related, alternative approach that allows us to establish a connection between the rejection parameters and various metrics such as sensitivity and specificity in the case of binary classification.


\subsection{Adaptive Model Selection for Regression}\label{sec:oms-reg}
In the context of regression, we assume we have some input data  $\{x_i\}_{i=1}^n \in \mathcal{X}^n$ and output data $\{y_i\}_{i=1}^n \in \mathbb{R}^n$.  We assume access to a collection of $m$ constituent regression models $\mathbf{h} = [h_1, \dots, h_m]$, such that $h_i : \mathcal{X} \rightarrow \mathbb{R}$. Similar to classification, given a fixed, finite set of weights $\mathbf{W} \subset \mathbb{R}^m$ we want to learn a function $f: \mathcal{X} \rightarrow \mathbf{W}$ such that for any $x \in \mathcal{X}$ (or $z \in \mathcal{Z}$), $f$ selects the best ensemble weights $\mathbf{w} \in \mathbf{W}$ such that $\mathbf{w}^T\mathbf{h}$ makes the best prediction on $x$. For this work, we measure the reward (or loss) using the squared error:
\begin{equation}
    R_{SE}(x_j, y_j, h_i) = (h_i(x_j) - y_j)^2,
\end{equation}
or equivalently for an ensemble $\mathbf{w} \in \mathbf{W}$,
\begin{equation}
    R_{SE}(x_j, y_j, \mathbf{w}, \mathbf{h}) = (\mathbf{w}^T\mathbf{h}(x_j) - y_j)^2.
\end{equation}
Again, for this work, we assume $\mathbf{e}_i \in \mathbf{W}$ for all $i \in [m]$. In this case, it is natural to frame the problem in terms of minimizing the reward. 
To our benefit, the reward function we define for regression is directly connected to the metrics used to evaluate regression models, namely mean squared error (MSE). That is, since our formulation uses total reward to select a model prescription in each leaf, the model that is prescribed in a given leaf is the one that has the minimum squared error over the samples that fall into that leaf. Minimizing the total reward also means that the resulting tree is directly optimizing to minimize the MSE. While we focus on squared error for this work, note that we could also easily measure reward in terms of absolute error if that was our metric of interest. 

Incorporating rejection learning for regression is simpler than in the case of classification. In this setting, we benefit from the reward function being directly connected to the evaluation metric of interest. Further, since the output of the models is one-dimensional, we only have to focus on a one-dimensional parameterization, as opposed to a $K$-dimensional parameterization for the multi-class classification setting with $K$ classes. We can introduce parameterized rejection by adding a dummy rejection model $h_r$ that is always a squared (or absolute) distance of $\alpha \in \mathbb{R}_{\geq 0}$ from any $x \in \mathcal{X}$ by fixing a constant reward
$$R(x_j, y_j, h_r) = \alpha.$$
 We treat $\alpha$ as our rejection threshold parameter. We then extend our action space to $\mathbf{\bar W} = \mathbf{W} \cup \{h_r\}$. With the rejection model added, we can again learn a policy tree $T_{O}$ that minimizes the reward. Then, for any leaf $l \in T_{O}$, the model prescribed in $l$ must have an MSE less than $\alpha$ over the data that falls into $l$. Otherwise, the leaf prescription will be to reject making a prediction. Therefore, rejection learning can be viewed as enforcing a maximum MSE over the training dataset. Similar to classification, for regression, as $\alpha \rightarrow 0$, we expect any policy model, including OP\textsuperscript{2}T, to converge to prescribing rejection always. We formalize this notion in Proposition \ref{prop:rejection-converge}.

\subsection{Alternative Approaches}\label{sec:alternative-approaches}
Having introduced our OP\textsuperscript{2}T approach, we now describe two alternative approaches, which we evaluate and compare against our method in Section \ref{sec:results}. Central to our work is the idea of using a prescriptive framework, measuring the relative reward of different models and ensembles, for predictive ML. It is natural then to ask what the benefit is of taking a prescriptive approach. A simpler alternative would be to instead treat the model prescription problem as a multi-class classification problem. In this setting, we have $M$ classes corresponding to each constituent model and ensemble, and each sample is labeled by the model or ensemble with the closest prediction to the true target. Formally, for models $[h_1,\dots,h_M]$, and data $\{(x_i,y_i)\}_{i=1}^n$, we assign sample $x_i$ a label $q_i \in [M]$ as
\begin{equation}
    q_i = \argmin_{i \in [M]}|y_i - h_i(x_i)|,
\end{equation}
where we break ties arbitrarily (e.g. smallest index). Then the equivalent of our approach in this setting would be a decision tree for classification, fit on the dataset $\{(x_i, q_i)\}_{i=1}^n$, or $\{(z_i, q_i)\}_{i=1}^n$ in case we are learning a policy over a modified feature space $\mathcal{Z}$. We refer to such a classification tree as a \textbf{Meta-Tree}. The Meta-Tree approach is a close comparison to our OP\textsuperscript{2}T method, as both learn over the same hypothesis space. That is, both approaches learn decision trees over the same feature space, with an equivalence between the $M$ classes and the prescriptions. The difference lies in how the tree is learned, and we demonstrate both theoretically in Section \ref{sec:theory} and empirically in Section \ref{sec:results} that this difference can lead to the resulting OP\textsuperscript{2}Ts outperforming the Meta-Trees. 

In addition to introducing a prescriptive framework to the problem of adaptive model selection, our approach yields interpretable policies that can be used to gain insights regarding existing predictive models. We therefore want to compare our approach against a related, black-box method. As we are focused on learning policies over structured data, it is natural to consider boosted trees as an alternative, black-box approach. Specifically, in the same vein as the Meta-Tree approach, we train boosted tree models, using XGBoost \citep{xgboost}, on the multi-class data $\{(x_i, q_i)\}_{i=1}^n$, or $\{(z_i, q_i)\}_{i=1}^n$. We refer to this approach as \textbf{Meta-XGB}. Since boosted tree models are more expressive than individual decision trees, this allows us to evaluate further the power of incorporating the relative rewards of each model and ensemble in a prescriptive framework. 

\section{Theoretical Insights}\label{sec:theory}
In this section, we formalize some of the concepts introduced in this work and provide theoretical insights for our approach. We discuss sufficient conditions under which OP\textsuperscript{2}T models can learn policies that improve over the best model in hindsight, the setting where OP\textsuperscript{2}Ts can outperform the Meta-Tree and Meta-XGB approaches, and the impact of rejection learning on the resulting policies. Throughout this section, we consider a fixed set of $m$ constituent models $\mathcal{H} = \{h_1, \dots, h_m\}$. We do not consider model ensembles separately, as any model ensemble can be represented as another constituent model, and this simplifies our notation. Our setup consists of an original feature space $\mathcal{X}$ over which the constituent models $\mathcal{H}$ are defined and a secondary feature space $\mathcal{Z}$ over which some policy model $\pi$ is defined.  We first introduce the following definition:
\begin{definition}
\textbf{(Informative Policy)} For some policy function $\pi: \mathcal{Z} \rightarrow \mathcal{H}$, function $g: \mathcal{X} \rightarrow \mathcal{Z}$, data distribution  $\mathcal{D}$ over $\mathcal{X}$, true target function $f^*: \mathcal{X} \rightarrow \mathcal{Y}$, and set of constituent models $H = [h_1,\dots,h_m]$, we refer to the policy as \textbf{informative} if
\begin{equation}
    E_{\mathcal{D}}[R(x, f^*(x), \pi(g(x)))] > \max_{h \in \mathcal{H}} E_{\mathcal{D}}[R(x, f^*(x), h)].
\end{equation}
\end{definition}
A policy is \textbf{informative} if it finds a meaningful partition of the space $\mathcal{Z}$ such that the expected reward under $\pi$ is greater than the expected reward for the best single constituent model. The function $g$ defines the relationship between samples in the original features space $\mathcal{X}$ and samples in the feature space $\mathcal{Z}$. This definition is helpful for reasoning about the potential effectiveness of adaptive model selection, regardless of the specific approach.

\subsection{Sufficient Conditions for Learning an Informative Policy}

Our first goal is to determine the conditions under which an informative policy exists. These conditions help us understand when an adaptive policy can outperform static model selection. Intuitively, an informative policy exists when there are separate regions of the feature space over which different models perform the best in terms of relative reward. To formalize this setting, we introduce the following definition:
\begin{definition}
    \textbf{(Dominated Subspace)} Let $\mathcal{Z} \subseteq \mathbb{R}^d$ be the feature space used to learn some policy function $\pi$, $\mathcal{D}$ some distribution over $\mathcal{X}$, $v: \mathcal{Z} \rightarrow 2^{\mathcal{X}}$ a function mapping elements of $\mathcal{Z}$ to disjoint subsets of the original feature space $\mathcal{X}$, and $f^*: \mathcal{X} \rightarrow \mathcal{Y}$ be the true target function. Suppose there are $m$ constituent models $\mathcal{H} = \{h_1,\dots,h_m\}$. Let us define the function $g: \mathcal{X} \rightarrow \mathcal{Z}$ as the unique function such that for all $x \in \mathcal{X}$, we have $x \in v(g(x))$. We refer to a connected subspace $A \subset \mathcal{Z}$ as a \textbf{dominated subspace} if $P_{\mathcal{D}}(\cup_{z \in A}v(z)) \in (0, 1)$ and there exists a model $h_A^* \in \mathcal{H}$ such that for all $h \in \mathcal{H} \setminus h_A^*$,  $$E_{\mathcal{D}}[R(x,f^*(x), h_A^*)\cdot \mathbbm{1}\{g(x) \in A\}] > E_{\mathcal{D}}[R(x, f^*(x), h)\cdot\mathbbm{1}\{g(x) \in A\}].$$ 
\end{definition}
 A \textbf{dominated subspace} is any connected subspace over which there is a single, best-performing model in expectation. Since we are reasoning about multiple reward surfaces jointly, it is natural to consider partitioning the feature space into these dominated subspaces. The functions $g$ and $v$ define the relationship between the feature spaces $\mathcal{X}$ and $\mathcal{Z}$. Notice that we assume the function $v$ maps to disjoint subsets of $\mathcal{X}$. That is, for all $z_1, z_2 \in \mathcal{Z}$ such that $z_1 \neq z_2$, $v(z_1) \cap v(z_2) = \emptyset$. This assumption gives $g$ the property of uniqueness. In many settings, this many-to-one relationship between $\mathcal{X}$ and $\mathcal{Z}$ is a reasonable assumption. For example, suppose $\mathcal{X}$ is unstructured language data (sequences of characters up to some maximum length) and $\mathcal{Z}$ is structured meta-data about this language data, such as the length of the sequence or specific character counts. Then an element $z \in \mathcal{Z}$ can map to many sequences in $\mathcal{X}$, but each sequence $x \in \mathcal{X}$ has one specific element in $\mathcal{Z}$ corresponding to its meta-data. Further, the elements of $\mathcal{Z}$ map to disjoint subsets of $\mathcal{X}$. Another setting with this property is when $\mathcal{Z}$ is the outputs of the constituent models $H$. With these definitions, we can then formulate the following proposition.
\begin{proposition}\label{prop:rewards}
     Let $\mathcal{Z} \subseteq \mathbb{R}^d$ be the feature space used to learn some policy function $\pi$, $\mathcal{D}$ some distribution over $\mathcal{X}$, a function $v: \mathcal{Z} \rightarrow 2^{\mathcal{X}}$ mapping elements of $\mathcal{Z}$ to disjoint subsets of the original feature space $\mathcal{X}$, and the unique function $g: \mathcal{X} \rightarrow \mathcal{Z}$ such that for all $x \in \mathcal{X}$, we have $x \in v(g(x))$. There exists an \textbf{informative policy} if there exist two disjoint dominated subspaces $A,B \subset \mathcal{Z}$ with corresponding models $h_A^*$ and $h_B^*$, such that $h_A^* \neq h_B^*$.
\end{proposition}

The proof is given in Appendix \ref{proof:rewards}. 
This proposition gives us clear sufficient conditions for the existence of an informative policy. Namely, given some assumptions, if there exist at least two disjoint dominated subspaces in $\mathcal{Z}$ corresponding to different constituent models, then an informative policy exists.
In addition to this result, it is clear that if there exists a model $h^* \in H$ such that $R(x, f^*(x), h^*) \geq \max_{h \in H} R(x, f^*(x), h)$ for all $x \in \mathcal{X}$, then an informative policy cannot be found, as one model strictly dominates the rest. 

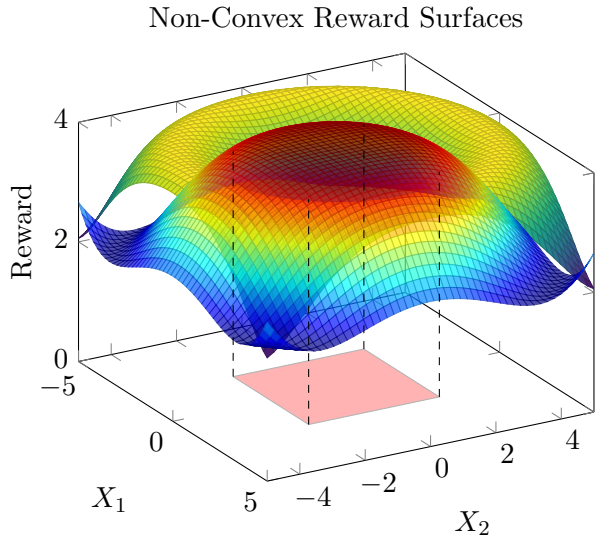
\begin{figure}[h!]
\begin{center}
\begin{tikzpicture}
  \begin{axis}[
    view={60}{30},
    xlabel=$X_1$, ylabel=$X_2$, zlabel=Reward,
    title={Non-Convex Reward Surfaces},
    legend style={at={(1.05,1)},anchor=north west},
    zmin=0, zmax=4
  ]
    \addplot3[
      surf,
      colormap/viridis,
      domain=-5:5,
      samples=50,
      opacity=1
    ]
    {sin(5*(x^2 + y^2))+3};

    \addplot3[
      surf,
      colormap/jet,
      domain=-5:5,
      samples=50,
      opacity=0.7
    ]
    {cos(5*(x^2 + y^2))+3};

    \addplot3 [
      samples=50,
      domain=-5:5,
      opacity=0.3,
      fill=red
    ] 
    coordinates {
      (-2,-2,0) (2,-2,0) (2,2,0) (-2,2,0) (-2,-2,0)
    };
     Dashed vertical lines from the corners
    \addplot3 [
      dashed,
      color=black,
    ] coordinates {
      (-2, -2, 0) (-2, -2, {cos(40)+3})
    };
    \addplot3 [
      dashed,
      color=black,
    ] coordinates {
      (2, -2, 0) (2, -2, {cos(40)+3})
    };
    \addplot3 [
      dashed,
      color=black,
    ] coordinates {
      (2, 2, 0) (2, 2, {cos(40)+3})
    };
    \addplot3 [
      dashed,
      color=black,
    ] coordinates {
      (-2, 2, 0) (-2, 2, {cos(40)+3})
    };
  \end{axis}
\end{tikzpicture}
\end{center}
    \caption{An example of two non-convex reward surfaces with a square dominated subspace.}
    \label{fig:dominated-subspace}
\end{figure}

We can apply this result to reason about the quality of OP\textsuperscript{2}T models. For example, when $\mathcal{Z} = \mathbb{R}$, connected subspaces are intervals $(a, b)$. Any such interval can be expressed with a depth-2 tree using parallel splits. Further, any pair of disjoint intervals, $(a,b)$ and $(c,d)$, can be expressed with a depth-3 tree using parallel splits. Therefore, it follows from Proposition \ref{prop:rewards} that if there exist two disjoint dominated subspaces in $\mathcal{Z} = \mathbb{R}$, an informative policy can be learned by a depth-3 OP\textsuperscript{2}T. More generally, it follows that if there exist two disjoint dominated subspaces that can be formed by parallel splits of a depth $D_{max}$ decision tree, then OP\textsuperscript{2}Ts can recover an informative policy. Note that the reward surfaces can be highly nonlinear as long as there exist dominated subspaces that can be defined by parallel splits generating an informative policy. For example, in Figure \ref{fig:dominated-subspace}, we show two non-convex reward surfaces, $f_1(X_1,X_2) = \sin(5(X_1^2+X_2^2))+3$ and $f_2(X_1,X_2) = \cos(5(X_1^2+X_2^2))+3$. Despite the surfaces being non-convex, there exists a dominated subspace that can be constructed by a depth-4 decision tree with axis-aligned splits.

\subsection{The Advantage of Prescription over Prediction}\label{sec:theory-prescription}
So far, we have reasoned about the combined reward surfaces for the constituent models and conditions for learning an informative policy with decision trees. The next question is then, what is the benefit of our prescriptive approach? In Section \ref{sec:alternative-approaches}, we introduce two alternative approaches which we refer to as Meta-Tree and Meta-XGB. These approaches do not consider the relative rewards of the constituent models. Instead, the problem is framed as a multi-class classification task, where each class corresponds to a different model or ensemble. Specifically, for models and ensembles $[h_1,\dots,h_M]$, and data $\{(x_i,y_i)\}_{i=1}^n$, we assign sample $x_i$ a label $q_i \in [M]$ as
\begin{equation}
    q_i = \argmin_{i \in [M]}|y_i - h_i(x_i)|,
\end{equation}
where we break ties arbitrarily (e.g. smallest index). These methods are simpler to implement, so it is important to quantify the advantage of using our prescriptive approach instead. To answer this question, we have the following result. 

\begin{proposition}\label{prop-op2t-vs-meta-tree}
    Under reward maximization, suppose we have a set of constituent models $\mathcal{H}=\{h_1,\dots,h_m\}$, data $\{(x_i, y_i)\}_{i=1}^n \in \mathcal{X}^n$, function $g: \mathcal{X} \rightarrow \mathcal{Z}$, and a reward function $R: \mathcal{X} \times \mathcal{Y} \times \mathcal{H} \rightarrow \mathbb{R}$ with either no finite lower bound or upper bound. Then the difference in total reward, $\sum_{i=1}^n R(x_i, y_i, T_O(g(x_i)) - \sum_{i=1}^n R(x_i, y_i, T_M(g(x_i))$, between the corresponding OP\textsuperscript{2}T, $T_O$, and a Meta-Tree, $T_M$, \textbf{can be arbitrarily large}. For a fixed dataset, if we define $R_{max} = \max_{i \in [n], (j,k) \in [m]} |R(x_i, y_i, h_j) - R(x_i, y_i, h_k)|$, we have the upper bound
    $$\frac{1}{n}\sum_{i=1}^n R(x_i, y_i, T_O(g(x_i)) - \frac{1}{n}\sum_{i=1}^n R(x_i, y_i, T_M(g(x_i)) \leq \frac{m-1}{m}R_{max}.$$
    The same result holds for the corresponding Meta-XGB model.
\end{proposition}

 The proof is given Appendix \ref{appx-op2t-vs-meta-tree}. The key insight from the proof is that the OP\textsuperscript{2}T approach can significantly outperform the Meta-Tree approach when the data $\{(x_i, q_i)\}_{i=1}^n$ is not separable such that we are in the non-realizable setting with the decision tree hypothesis class. In these settings, a trade-off must be made between model prescriptions. The OP\textsuperscript{2}T approach considers more information regarding the relative reward of each constituent model, resulting in a prescription that can be arbitrarily better than the one made by the Meta-Tree approach. Note that the condition that $R$ must have either no finite lower bound or upper bound is satisfied by both $R_{CE}$ and $R_{SE}$. Importantly, and more practically, the difference in expected reward between the OP\textsuperscript{2}T approach and the alternative approaches is bounded by the magnitude of the relative difference in rewards between the constituent models. When the relative difference in reward between the models and ensembles is small, we can expect minimal differences between the approaches. However, in the non-realizable setting, for large differences in reward, we expect the OP\textsuperscript{2}T to outperform the alternative approaches. While the same results hold for Meta-XGB, we note that boosted trees are a more expressive model class, such that they can separate data better than decision trees. Therefore, we expect a smaller gap between OP\textsuperscript{2}Ts and the Meta-XGB approach in terms of total reward. We give an example highlighting this phenomenon in Section \ref{sec:projectile}, in addition to showing a similar result on a 1-D toy dataset in Section \ref{sec:synthetic}.

\subsection{The Quality of OP\textsuperscript{2}T Policies and the Impact of Rejection Learning}\label{sec:theory-rejection}
In this section, we investigate the impact of rejection learning on adaptive model selection. We do so through the lens of counting the number of dominated subspaces. Intuitively, as the number of regions with different dominant models increases, more expressive policy functions are required to partition the feature space accordingly. Conversely, rejection learning can serve as a mechanism to decrease the number of dominated subspaces, specifically those with low expected reward. To help reason about the number of dominated subspaces, we introduce the following definition.
\begin{definition}
    \textbf{(Maximal Dominated Subspace)} A dominated subspace $A$ is a maximal dominated subspace if for all dominated subspaces $B$ such that $A \subset B$, $h^*_A \neq h^*_B$.
\end{definition}
It is important to distinguish these dominated subspaces for the purpose of counting, as many dominated subspaces, such as an interval $(a, b) \subset \mathbb{R}$, $a < b$, may imply the existence of infinite, nested dominated subspaces, all corresponding to the same constituent model. Even with these conditions, there may also exist many, or even infinite, maximal dominated subspaces. In the example depicted in Figure \ref{fig:dominated-subspace}, if we take the feature space to be all of $\mathbb{R}^2$, there are infinite maximal dominated subspaces. Practically, we may not expect such a case to occur in reality, but given the likely non-convexity of reward surfaces, it is possible that there are many maximal dominated subspaces. Given that an informative policy exists, the quality of our tree-based approach depends on the number of maximal dominated subspaces and the shape and dimension of these subspaces. In the case that the set of maximal dominated subspaces, $S$, is small (i.e., $|S| << 2^{D_{max}}$) and each subspace $s \in S$ is representable by $D_{max}$ parallel splits, an OP\textsuperscript{2}T model of depth $D_{max}$ could learn a policy that approximately recovers all of these subspaces. Conversely, if there are many maximal dominated subspaces that are high dimensional and nonlinear, there will be a gap in expected reward between the policy learned by an OP\textsuperscript{2}T and the true optimal policy function. We note, however, that extending our formulation to include hyperplane splits is simple, and we implemented this version of the model, although empirically we did not see a benefit. These hyperplanes splits would allow the OP\textsuperscript{2}Ts to represent more complex dominated subspaces with fewer splits.

The inclusion of a rejection model can simplify the task of learning informative policies while adding a form of regularization.
Formally, we make the following statement:
\begin{proposition}
    Under reward maximization, given some data distribution $\mathcal{D}$ over $\mathcal{X}$ and reward function $R: \mathcal{X} \times \mathcal{Y} \times \mathcal{H} \rightarrow \mathbb{R}$ with finite upper bound $R_{ub}$ such that $R(\cdot) \leq R_{ub}$, and constituent models $H = \{h_1,\dots,h_m\}$, $h_i: \mathcal{X} \rightarrow \mathcal{Y}$, and true target function $f^*: \mathcal{X} \rightarrow \mathcal{Y}$, suppose there exist finite maximal dominated subspaces. Further, let us define a rejection model $h_r^n$ such that $R(x,y, h_r^n) = \alpha_n$.  We denote the set of maximal dominated subspaces, not including the regions dominated by the rejection model $h_r^n$, by $S^n = \{s_1^n,\dots, s_q^n\}$. Given an increasing sequence of rejection parameters $\{\alpha_n\}$ such that $\alpha_n \rightarrow R_{ub}$, we have the following properties:
    \begin{enumerate}
        \item If $P_D(R(x,f^*(x), h_i) = R_{ub}) = 0$ for all $i \in [m]$, then $|S^n| \rightarrow 0$ as $n \rightarrow \infty$.
        \item For all dominated regions $s_i^n$, $s_j^{n+1}$ such that $s_j^{n+1} \subseteq s_i^{n}$, we have $P_D(s_i^n) \geq P_D(s_j^{n+1})$ and $E_D[R(x,f^*(x),h^*_{s_j^{n+1}}) | x \in s_j^{n+1}] \geq E_D[R(x,f^*(x),h^*_{s_j^n}) | x \in s_j^{n+1}]$.
    \end{enumerate}
    
\end{proposition}\label{prop:rejection-converge}
The proof of these properties is brief and is given in Appendix \ref{appx-rejection-properties}. These properties give us a sense of the impact of rejection, specifically that increasing the reward for rejection reduces the total number of maximal dominated subspaces in the limit while also making these subspaces smaller and increasing their expected reward. For example, consider a pair of models that generate reward surfaces such that there are many maximal dominated subspaces, far greater than can be partitioned by a tree of any reasonable depth. By increasing the reward for rejection, we will likely reduce the number of maximal dominated subspaces, allowing the tree to focus on identifying fewer, significant partitions with high expected reward.

\section{Experiments}\label{sec:results}
We evaluate our methodology on a variety of real-world datasets with both structured and unstructured data, including regression and classification tasks. For all experiments, we perform a hyperparameter search over the tree complexity penalty $\lambda$, max depth $D_{max}$, and minimum number of samples per leaf $c_{min}$. To preserve interpretability, we limit the max depth to at most $D_{max}=10$. We then fit our OP\textsuperscript{2}T models via $k$-fold cross validation (for $k \in \{3,5\}$).

\subsection{Benchmarks}
To evaluate our approach, we benchmark against two other tree-based methods. Namely, we frame the problem of selecting a constituent model (or ensemble) as a multi-class classification problem, where each sample is assigned a label corresponding to the best-performing model on that sample, measured by the absolute distance between the model score and the true target. We defined this formulation in Section \ref{sec:theory}. With each sample labeled with a constituent model, we then fit a standard CART decision tree, as in \cite{breiman_classification_1984}, in addition to a boosted tree using XGBoost, on this dataset. For both approaches, we perform hyperparameter tuning over the max depth and the minimum number of samples per leaf. In addition, for the XGBoost models, we tune the total number of estimators. We refer to the CART-based method as Meta-Tree, and the XGBoost method as Meta-XGB. We chose these as benchmark models to investigate the impact of taking a prescriptive approach to this problem. Further, the Meta-Tree comes from the same model class as our OP\textsuperscript{2}Ts, the set of parallel-split decision trees, and share the same property of interpretability. The Meta-XGB model further allows us to evaluate the potential impact of taking a black-box, boosting approach. 

\subsection{A Toy 1-D Example}\label{sec:synthetic}
To validate our approach and demonstrate learning tree-based policies and the effects of varying rejection thresholds, we provide a simple synthetic experiment with 1-dimensional data. Specifically, we assume $\mathcal{X} = \mathbb{R}$, and there are two models, $M_1$ and $M_2$ with reward functions $R(M_1, x, y) = \exp(-\frac{1}{2}(x-4)^2)$ and  $R(M_2, x, y) = \exp(-\frac{1}{2}(x-8)^2)$. We then draw $n=500$ samples uniformly over the interval $[0,12]$ and fit our OP\textsuperscript{2}T model to this data. In addition, we add rejection models with $\alpha = 0.1$ and $\alpha = 0.3$. A visualization of the reward surfaces is given in Figure \ref{fig:gauss-reward-plot}, and a visualization of the policy trees generated is given in Figure \ref{fig:gauss-reward-tree-plot}. The resulting trees match the splits we would expect which are highlighted in the plot.

\begin{figure}[h!]
    \centering
    \includegraphics[width=0.6\linewidth]{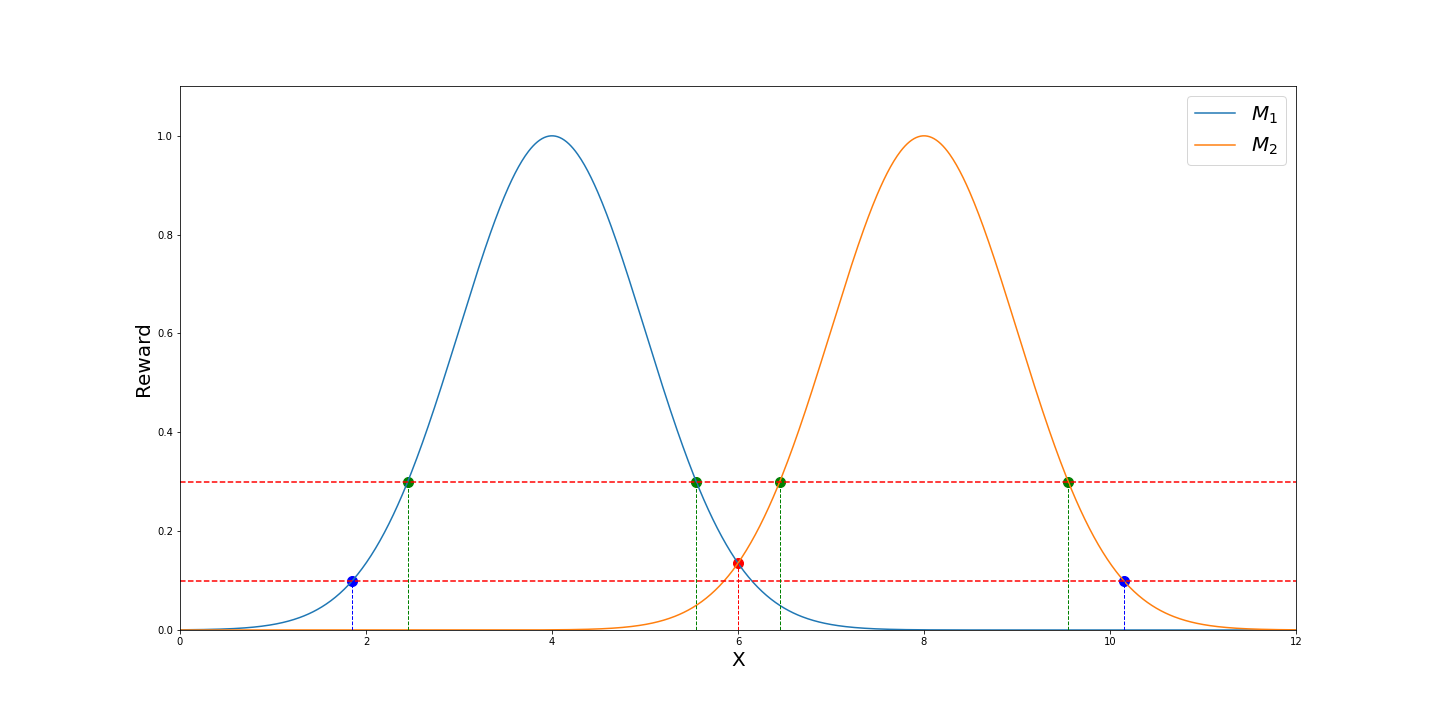}
    \caption{A simple 1-D example of synthetic model rewards with different rejection thresholds, denoted by the red dashed horizontal lines. In this case, the feature space $\mathcal{X} = [0,12]$ and the rejection parameters are $\alpha = 0.1$ and $\alpha = 0.3$.}
    \label{fig:gauss-reward-plot}
\end{figure}

\begin{figure}[h!]
    \centering
    \includegraphics[width=\linewidth]{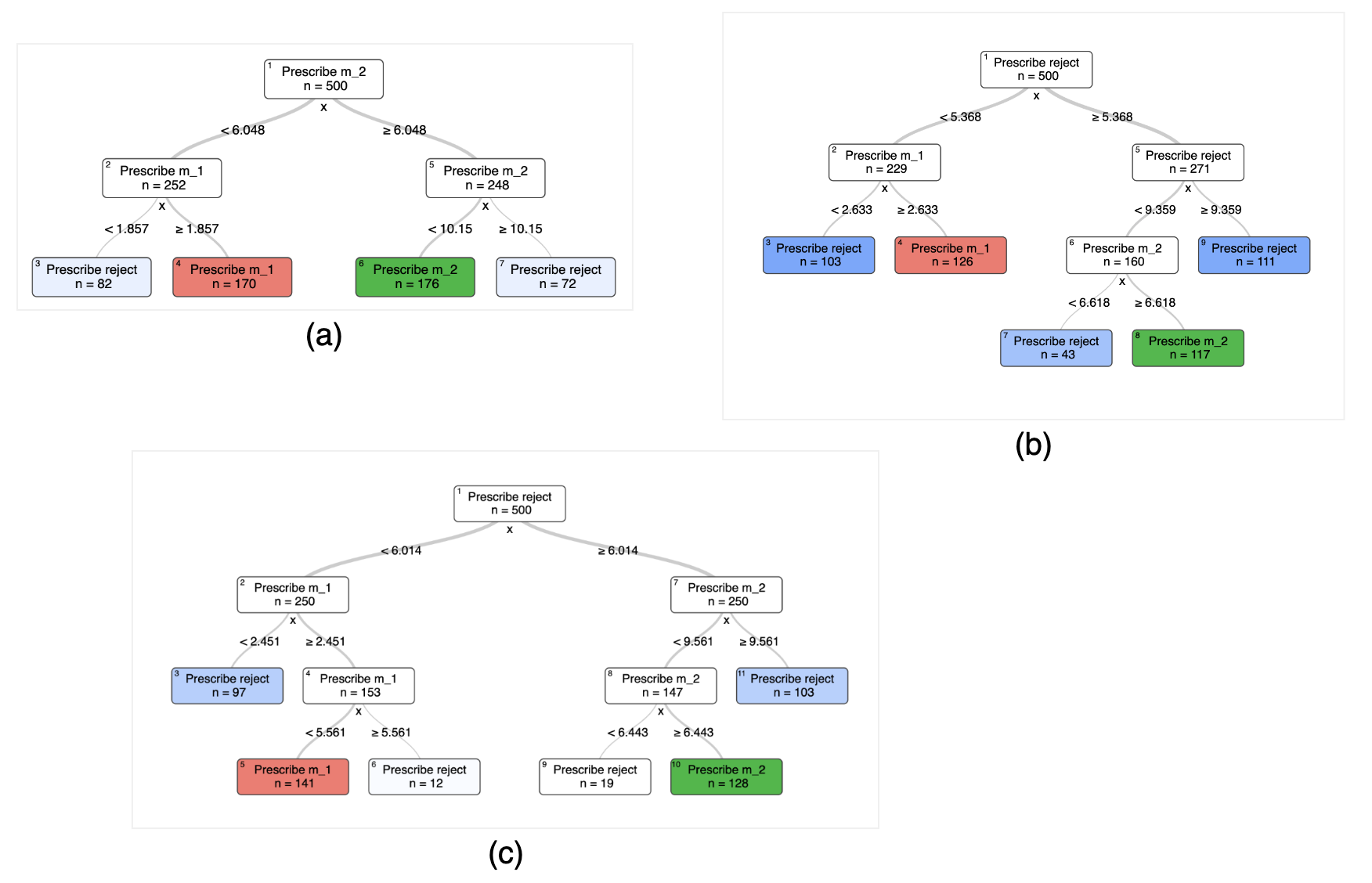}
    \caption{OP\textsuperscript{2}T models fit on the 1-D synthetic data described in Section \ref{sec:synthetic}. Tree (a) corresponds to no rejection, (b) to rejection with $\alpha = 0.1$, and (c) to rejection with $\alpha=0.3$.}
    \label{fig:gauss-reward-tree-plot}
\end{figure}

By making a small change to the previous example, we can also construct a setting that demonstrates the difference between learning an OP\textsuperscript{2}T and a Meta-Tree. Consider now two models, $M_1$ and $M_2$, over the feature space $\mathcal{X} = [0,18]$ with reward functions $R(M_1, x, y) = \exp(-\frac{1}{2}(x-4)^2) + 0.01$ and  $R(M_2, x, y) = \exp(-\frac{1}{2}(x-8)^2)$. This modification creates an interval, $[11, 18]$, over which model $M_1$ achieves a marginally higher reward than $M_2$. A visualization of the reward surfaces is given in Figure \ref{fig:gauss-reward-tail-plot}, shaded corresponding to the dominant model in each region. We sample $n=500$ points uniformly over $[0, 18]$ and generate a policy tree using our OP\textsuperscript{2}T approach, given in Figure \ref{fig:gauss-reward-tail-tree-plot}. Notice that in the right sub-tree, the prescription at the inner node is $M_2$, despite model $M_1$ yielding a greater reward than $M_2$ for the majority of the samples falling into this sub-tree. Were we to take a predictive, classification approach, and label each sample by the model that maximizes reward, the prediction at this node would be $M_1$. This would ignore the fact that the average reward for $M_2$ over the interval $[6,18]$ is significantly greater than $M_1$ (in this case, the difference in total reward is 2.27, and in expectation is 0.1856). We can see this difference in Figure \ref{fig:gauss-reward-tail-tree-plot}(a), where the Meta-Tree approach prescribes model $M_1$ at the root, left, and right subtrees. Therefore, while both trees perfectly separate the data, the depth-1 subtree of the OP\textsuperscript{2}T, starting from the root, achieves a total reward of 4.959, while the depth-1 subtree of the Meta-Tree achieves a total reward of 2.686. The result of this example is consistent with our analysis in Section \ref{sec:theory-prescription}. We also observe another, more dramatic example of this phenomenon in Section \ref{sec:projectile}.

\begin{figure}[h!]
    \centering
    \includegraphics[width=0.5\linewidth]{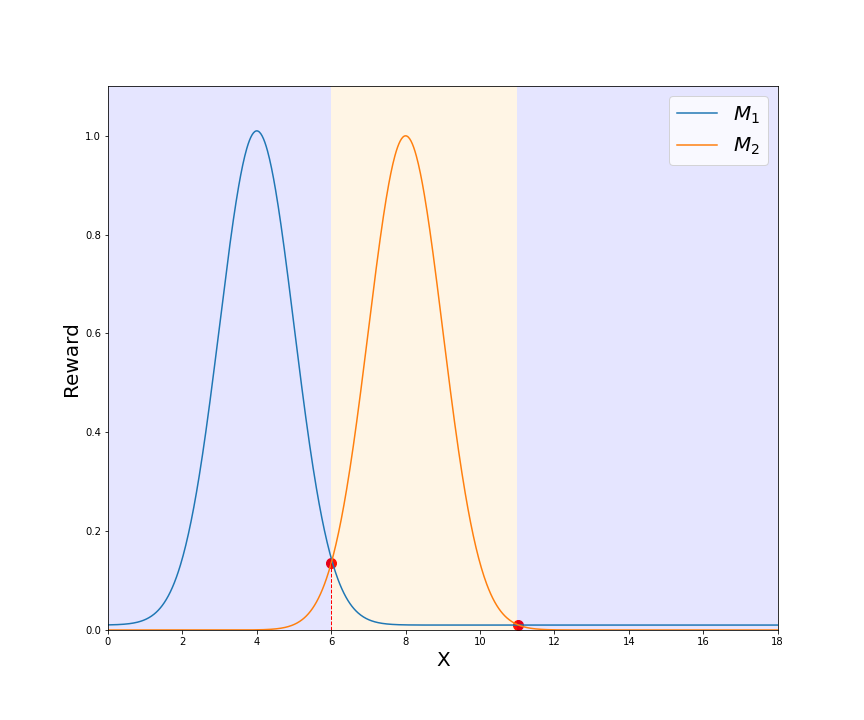}
    \caption{A simple 1-D example of synthetic model rewards demonstrating the potential difference between taking a prescriptive versus a predictive approach to model selection.}
    \label{fig:gauss-reward-tail-plot}
\end{figure}
\begin{figure}[h!]
\centering
    \subfloat[Meta-Tree]{\includegraphics[width=0.5\linewidth]{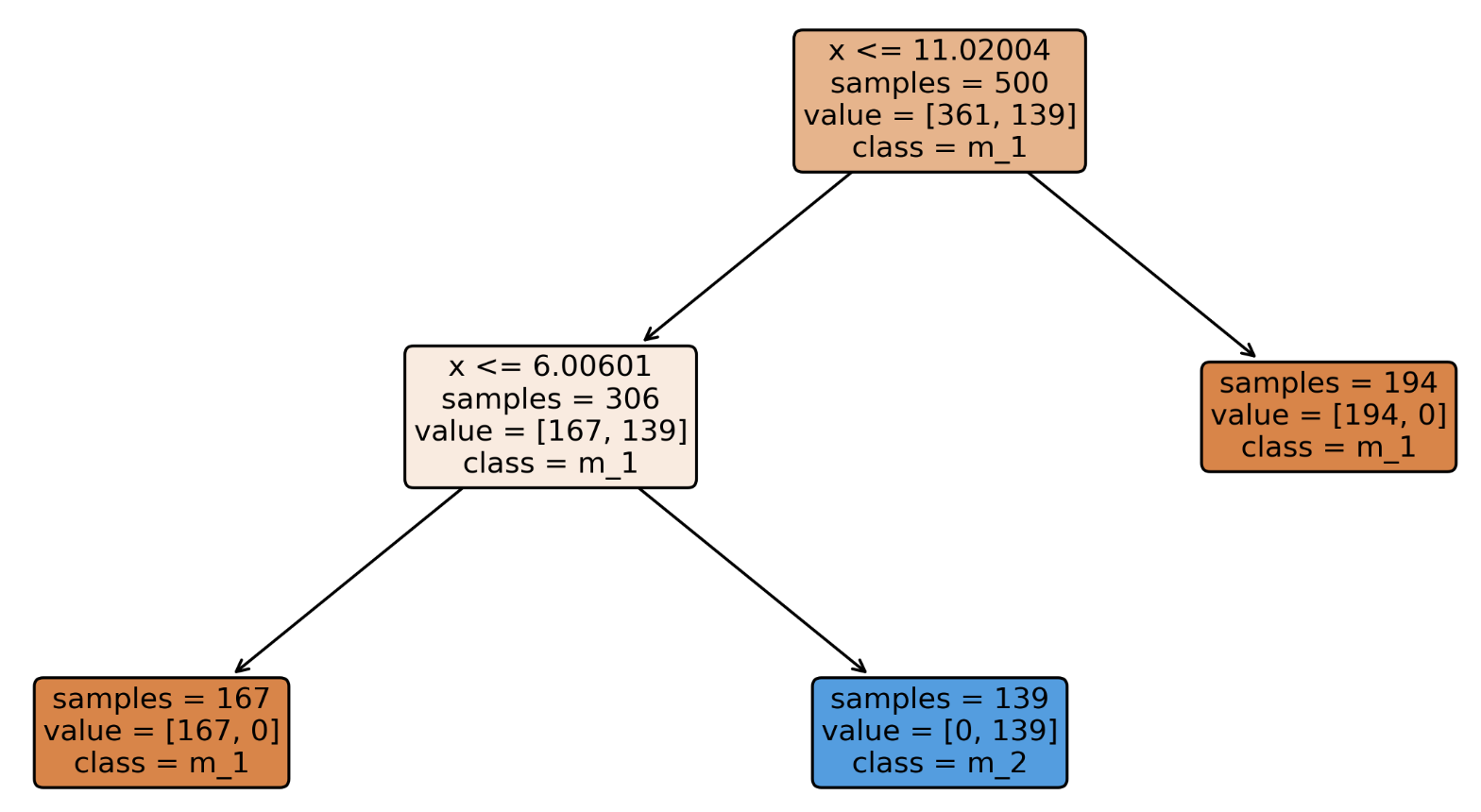}}
    \subfloat[OP\textsuperscript{2}T]{
    \includegraphics[width=0.4\linewidth]{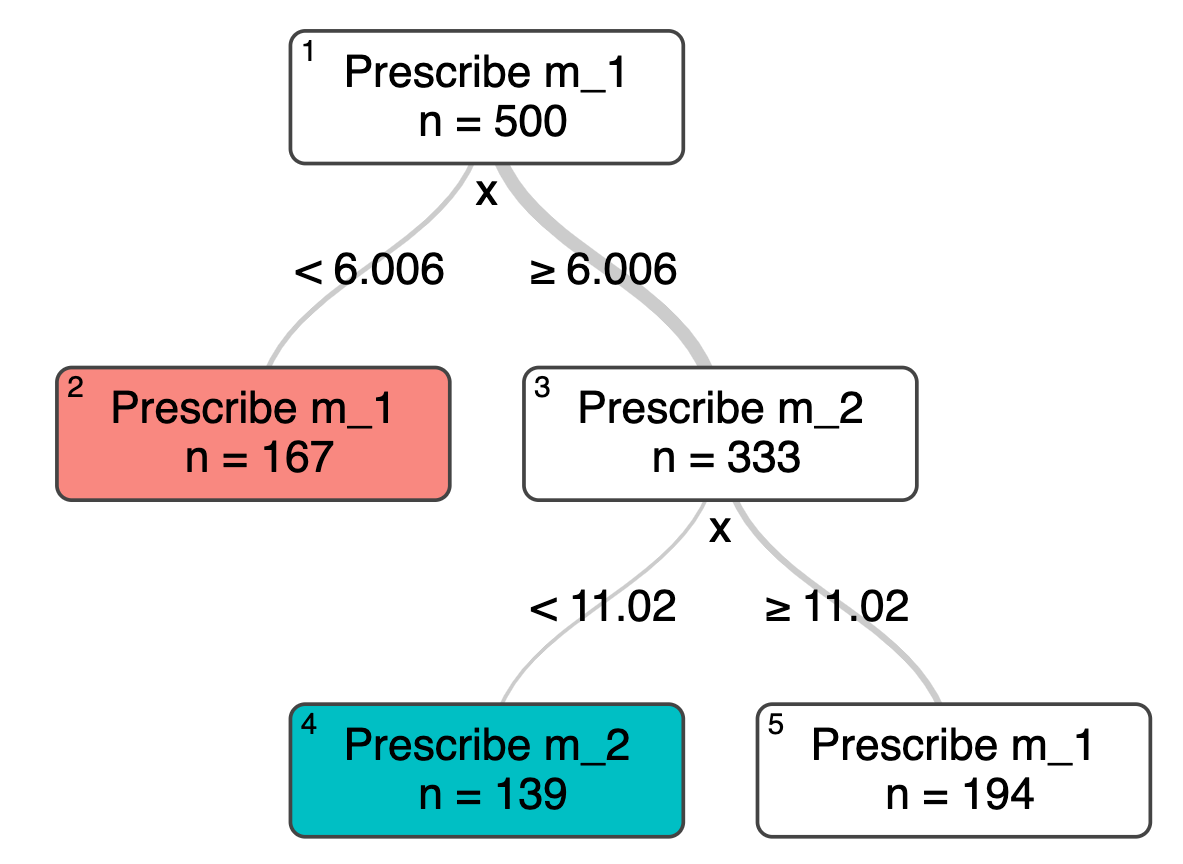}}
    \caption{OP\textsuperscript{2}T and Meta-Tree models fit on the 1-D synthetic data corresponding to the rewards shown in Figure \ref{fig:gauss-reward-tail-plot}.}
    \label{fig:gauss-reward-tail-tree-plot}
\end{figure}

\subsection{Concrete Compressive Strength}\label{sec:concrete}
As a baseline for evaluating our OP\textsuperscript{2}Ts on regression tasks, we apply our approach to the Concrete Compressive Strength dataset (n=1080, p=8), publicly available through the UCI ML Repository \citep{misc_concrete_compressive_strength_165}. We fit a boosted tree model (xgb), a random forest model (rf), a linear regression (lr) model, and an MLP model on the training data. We then create two ensembles: a mean ensemble and a ridge ensemble. The weights for the ridge ensemble are found by solving the optimization problem 
$$\mathbf{w}^* = \text{argmin}_{\mathbf{w} \in \mathbb{R}^m} \frac{1}{n}||\mathbf{wh}(\mathbf{X})^T - \mathbf{y}||_2^2 + \lambda ||\mathbf{w}||_2^2,$$
where $\mathbf{h}(\mathbf{X}) \in \mathbb{R}^{n \times m}$ are the predictions from the $m$ constituent models on some input data $\mathbf{X} \in \mathbb{R}^{n \times m}$ with $n$ samples and $\mathbf{y} \in \mathbb{R}^n$. This formulation has a closed-form solution,
$$\mathbf{w}^* = ((\mathbf{h}(\mathbf{X})\mathbf{h}(\mathbf{X})^T + \lambda \mathbf{I})^{-1}\mathbf{h}(\mathbf{X}))^T\mathbf{y}.$$
We fix $\lambda = 1$ for our experiments. Then, for our experiments, our actions are the set of weights $\mathbf{W} = (\mathbf{e}_1, \dots, \mathbf{e}_m, (\frac{1}{m}, \dots, \frac{1}{m}), \mathbf{w}^*)$. For rejection, we evaluate the threshold parameter $\alpha$ over the set $[100, 40, 30]$. The results of this experiment are shown given in Table \ref{tab:concrete_results}. Without rejection, the best constituent model (including the ensembles) achieved an out-of-sample MSE of 39.2, while the OP\textsuperscript{2}T without rejection achieved an out-of-sample MSE of 36.13, demonstrating that this approach can learn a simple, interpretable policy that generalizes well. We also observe that both the mean and ridge ensemble weights are prescribed in the trees, yielding an interpretable, adaptive model ensembling policy. Further, we gain some insight from the learned policy. In Figure \ref{fig:concrete-results-plt}, we provide a visualization of three OP\textsuperscript{2}T models fit with a max depth of $D_{\max}=3$, for ease of interpretability. We see that by partitioning the data based on the density of concrete in the cement mixture, splitting at 357.5 $kg/m^3$, the XGBoost model outperforms all other constituent models when the density is less than 357.5 $kg/m^3$, while the random forest model performs the best when the density is greater than or equal to 357.5 $kg/m^3$. Looking at the statistics for the dataset, we found that the 75th percentile for concrete density is 350 $kg/m^3$. Therefore, we can conclude that the learned policy is to use the random forest model when the cement density is high (above the 75th percentile), otherwise, it is best to use the XGBoost model.

\begin{table}[!h]
    \centering
    \begin{tabular}{|c|cc:cc:cc|}
        \hline
        \multirow{2}{*}{Model}  & \multicolumn{2}{c:}{No Rejection} & \multicolumn{2}{c:}{Rejection w/ $\alpha = 40$} & 
        \multicolumn{2}{c|}{Rejection w/ $\alpha = 30$} \\
        \cline{2-7}
        & MSE & Reject (\%)& MSE & Reject (\%)& MSE & Reject (\%)\\
        \hline
        Meta-XGB  & 39.69 & 0 & 37.32 & 18.93 & 33.58 & 43.20 \\
        Meta-Tree  & 39.57 & 0 & 42.82 & 16.02 & 37.35 &  37.86 \\
        OP\textsuperscript{2}T & 36.13 & 0 & 27.30 & 17.01 & 25.22 & 37.43 \\
        \hline
    \end{tabular}
    \caption{Out-of-sample MSE on the concrete compressive strength dataset across different thresholds for rejection.}
    \label{tab:concrete_results}
\end{table}

With the rejection option, we see that our OP\textsuperscript{2}T approach learns policies that generalize well. In each case, we can also interpret the policies that are learned, and we observe consistency across the policy trees. For example, in tree (b) we see that the left subtree learns a split similar to tree (a), prescribing the random forest model when cement density is high, and prescribing the XGBoost model otherwise. Something similar is learned in tree (c), in the left-most subtree, except when the cement density is high, the tree opts to reject rather than prescribing the random forest model. This is due to the reward for rejection being considerably higher in tree (c) than in trees (a) and (b). We can also observe that Trees (b) and (c) both learn to split on Blast Furnace Slag and Water density at very similar values. Further, both reject when  Blast Furnace Slag density is high (above the 60th percentile) and Water density is low (below the 50th percentile). Tree (d) also learns a similar split on Cement density, as well as on Blast Furnace Slag density, except it rejects any time the Blast Furnace Slag density is high. From these tree-based policies, we can identify clear and consistent partitions of the data within which different models perform well and partitions under which no models perform particularly well. To summarize, we can connect the results of this experiment directly back to the motivating questions we posed at the beginning of the paper.

\begin{figure}[h!]
    \centering
    \includegraphics[width=1\linewidth]{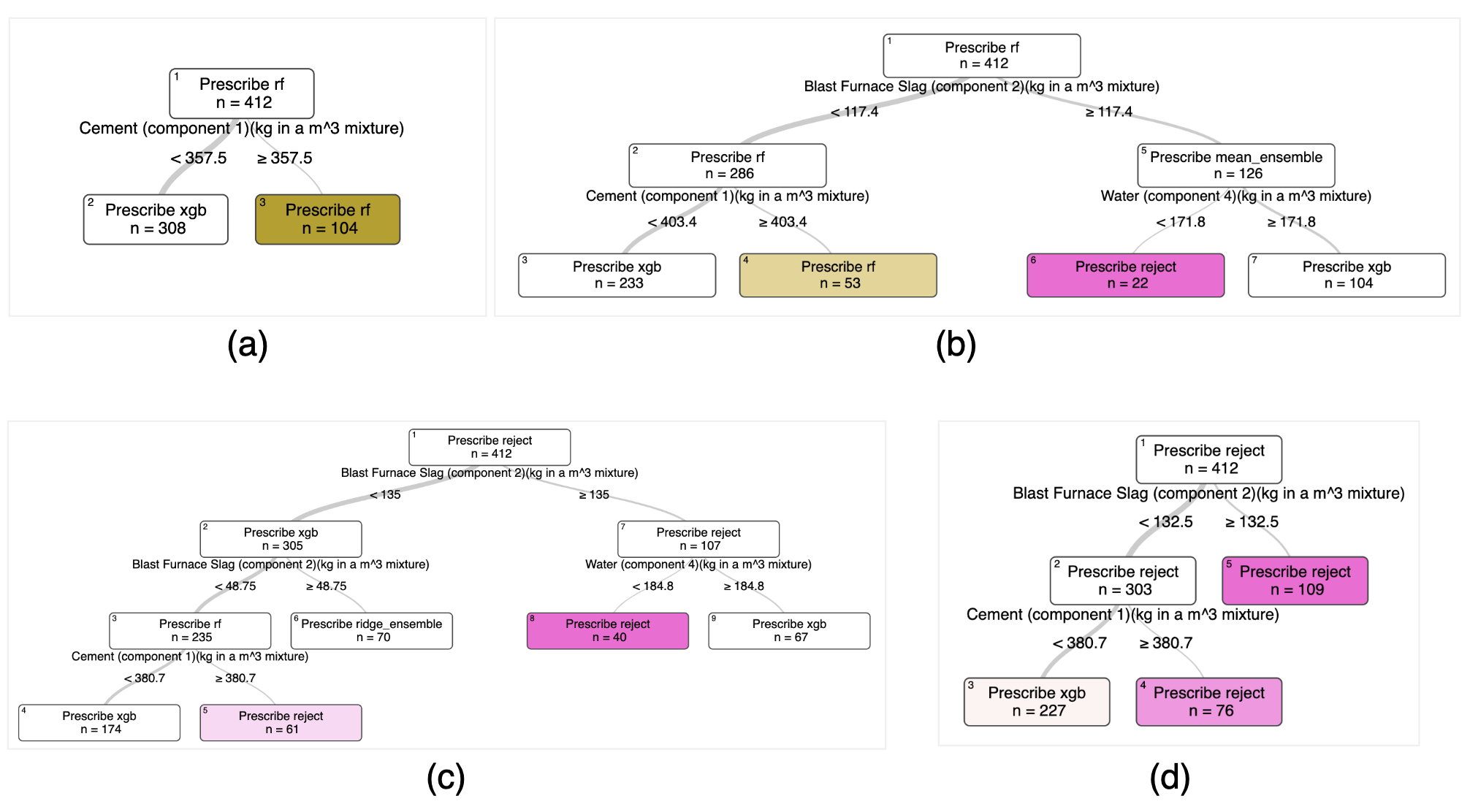}
    \caption{OP\textsuperscript{2}T models fit on the Concrete Compression dataset with a max depth of $d=3$ (for interpretability) and varying rewards for rejection. Tree (a) is fit without a rejection option, while trees (b), (c), and (d) are fit with rejection parameters $\alpha = [100, 40, 25]$ respectively.}
    \label{fig:concrete-results-plt}
\end{figure}

\noindent \textbf{Should we use a model ensemble? If so, when?} Yes, in some settings, the mean ensemble is the best, while in other cases the ridge ensemble is better. For example, when the blast furnace slag component is in the range $[49, 135]$, tree (c) in Figure \ref{fig:concrete-results-plt} identifies the ridge ensemble to be the best model. However, for the majority of samples, using either the XGBoost or random forest model alone is better than using one of the ensembles.

\noindent \textbf{When are the models likely to be error-prone?} In short, when there is a high concentration of the blast furnace slag or the cement component in the concrete mixture.

\noindent \textbf{Should we use an interpretable model?} In this case, when the interpretable model is linear regression, the answer is no. Across all rejection thresholds, the OP\textsuperscript{2}Ts do not identify any partition of the feature space where the linear regression model performs the best. It is possible that other interpretable models may perform better, and a practitioner may want to explore additional model classes if interpretability is critical.

\subsection{Projectile Motion with Drag}\label{sec:projectile}
Motivated by recent work that has successfully combined physics-based and ML models for predicting the dynamics of physical systems \citep{boussioux_hurricane_2022}, we develop a synthetic experiment to demonstrate how OP\textsuperscript{2}Ts can adaptively prescribe these models based on the conditions of the physical system, significantly reducing the overall error. The setup is as follows. Suppose we are launching a projectile in 2-D space with a fixed mass $m =1$ from ground level, at the origin $(0,0)$, with a launch angle $\theta$ and initial speed $v_0$. Further, we are going to assume that there is some air resistance parameterized by a scalar $c$ which we will soon define. We are interested in estimating the total ground distance covered by the projectile as a function of $\theta, v_0$ and $c$. Let $x_f$ denote the final position of the projectile when it hits the ground. In the absence of air resistance or other effects, this can be found to be
\begin{equation}\label{eq:physics-no-drag}
    \hat x_f = \frac{v_0\sin(2\theta)}{g}.
\end{equation}
Now, we incorporate a notion of drag by assuming that air resistance occurs in the opposite direction of the projectile's velocity, and its magnitude is directly proportional to the current speed. The equation of motion for the projectile is then given by
\begin{equation}
    \mathbf{F} = m\mathbf{a} = m\frac{d\mathbf{v}}{dt} = m\mathbf{g} - c\mathbf{v},
\end{equation}
where $\mathbf{g} = (0, -g)$ is the gravitational force, and $\mathbf{v}$ is the velocity of the projectile. With these dynamics defined, the vertical position of the projectile at time $t$ is given by
\begin{equation}\label{eq:physics-y}
    y(t) = \frac{v_t}{g}(v_0\sin(\theta) + v_t)(1 - \exp(-\frac{gt}{v_t})) - v_tt,
\end{equation}
where $v_t = mg/c$ is the terminal velocity. The horizontal position is then given by
\begin{equation}\label{eq:physics-x}
    x(t) = \frac{v_0v_t\cos(\theta)}{g}(1 - \exp(-\frac{gt}{v_t})).
\end{equation}
Then, by finding $t^* > 0$ such that $y(t^*) = 0$, we can find $x_f = x(t^*)$. We also notice that in the limit, as $t \rightarrow \infty$, we have
\begin{equation}\label{eq:physics-drag}
  \bar  x_f = \frac{v_0v_t\cos(\theta)}{g},
\end{equation}
which gives us an upper bound on the ground distance traveled by the projectile in the presence of this drag force. From this analysis, we now have two closed-formed estimates of $x_f$ as a function of $\theta, v_0$, and $c$, namely Eqs. \eqref{eq:physics-no-drag} and \eqref{eq:physics-drag}. For brevity, we refer to the estimate from \eqref{eq:physics-no-drag} as $\hat x_f$ and \eqref{eq:physics-drag} as $\bar x_f$. We want to create a pattern-based ML model that learns to predict $x_f$ using ``real-world'' samples. Assuming that, perhaps due to physical limitations, the maximum initial speed at which we can launch the projectile is $100 \text{ } m/s$, we can define our feature space as $\mathcal{X} = [0, 100] \times [0, \pi/2] \times [0,1]$, since $v_0 \in [0, 100]$, $\theta \in [0, \pi/2]$, and $c \in [0, 1]$. We then discretize the feature space and randomly sample $n=20k$ data points from $\mathcal{X}$. For each sample, we find the ``ground-truth'' value for $x_f$ using Eqs. \eqref{eq:physics-x} and \eqref{eq:physics-y}, and use these values as our target. Then, splitting the data into $50\%$ training, $25\%$ validation, and $25\%$ test sets, we fit an MLP model $h$ with hidden sizes $[16,32,16]$ on the training data. We chose the MLP model class because of its ability to model smooth, nonlinear functions. Our constituent models are then $\hat x_f$, $\bar x_f$, and $h$. For this experiment, we do not consider model ensembles, so we set $\mathbf{W} = \text{diag}(1,1,1)$. We fit our OP\textsuperscript{2}T model with the predictions from these constituent models on the validation set. The results, without the rejection option, are given in Table \ref{tab:physics_results}. 

\begin{table}[!h]\label{table:physics-no-reject}
    \centering
    \begin{tabular}{|c|c|}
    \hline
        Model & MSE\\
        \hline
        Physics (no drag) & $59098.38$\\
        Physics (drag) & $188948.88$\\
        MLP & $241.16$\\
        \hline
    \end{tabular}
    \caption{Out-of-sample MSE for the constituent models of the projectile motion dataset. Notice how the closed-form physics equations are extremely poor estimators on average over the entire feature space.}
\end{table}

\begin{table}[!h]
    \centering
    \begin{tabular}{|c|cc:cc:cc|}
        \hline
        \multirow{2}{*}{Model}  & \multicolumn{2}{c:}{No Rejection} & \multicolumn{2}{c:}{Rejection w/ $\alpha = 250$} & 
        \multicolumn{2}{c|}{Rejection w/ $\alpha = 50$} \\
        \cline{2-7}
        & MSE & Reject (\%) & MSE & Reject (\%)& MSE & Reject (\%) \\
        \hline
        Meta-XGB & 140.85 & 0 & 28.65 & 9.96 & 5.56 & 31.32 \\
        Meta-Tree  & 150.56 & 0 & 36.19 & 12.14 & 7.67 & 34.20 \\
        OP\textsuperscript{2}T & 141.74 & 0 & 31.53 & 11.54 & 6.63 & 32.8\\
        \hline
    \end{tabular}
    \caption{Out-of-sample MSE on the projectile motion dataset across different thresholds for rejection.}
    \label{tab:physics_results}
\end{table}

What is most surprising from these results is that despite the closed-form physics-based equations for $x_f$ performing terribly on average, as shown in Table \ref{table:physics-no-reject}, our OP\textsuperscript{2}T approach successfully uncovers the relatively small regions of $\mathcal{X}$ within which these estimates are very close to the ground-truth, and uses this to drastically reduce the MSE compared to the MLP model alone. Further, the OP\textsuperscript{2}T significantly outperforms the Meta-Tree, especially for shallow tree depths. It is also not obvious a priori that such regions would exist. Since the MLP model has been trained on data sampled across the entire feature space, it is possible for this model to fit all of the data quite well. In practice, however, in minimizing the average error, the MLP has made some trade-offs and cannot compete with the physics-based equations in some regions of $\mathcal{X}$. We also observe that for this example, the Meta-XGB approach yields the best performance in terms of MSE and percent of samples rejected. This agrees with our theoretical analysis in Sections \ref{sec:theory-prescription} and \ref{sec:theory-rejection}, as the boosted tree model is more capable of fitting complex, nonlinear boundaries, such that the relative rewards of the models in each partition is not significant. A primary conclusion of Proposition \ref{prop-op2t-vs-meta-tree} is that the OP\textsuperscript{2}T has an advantage in the non-realizable setting when we cannot perfectly separate the feature space by the best-performing model. In this case, the Meta-XGB approach is successful in learning more complex partitions of the feature space. However, the Meta-XGB approach lacks interpretability, so we cannot gain insights from the resulting boosted tree model. Standard explainability methods for boosted trees, such as plotting relative feature importance, are not particularly useful in this setting either. We aim to answer questions regarding which model to use, and when, not which features are generally important.

\begin{figure}[h!]
    \centering
    \includegraphics[width=1\linewidth]{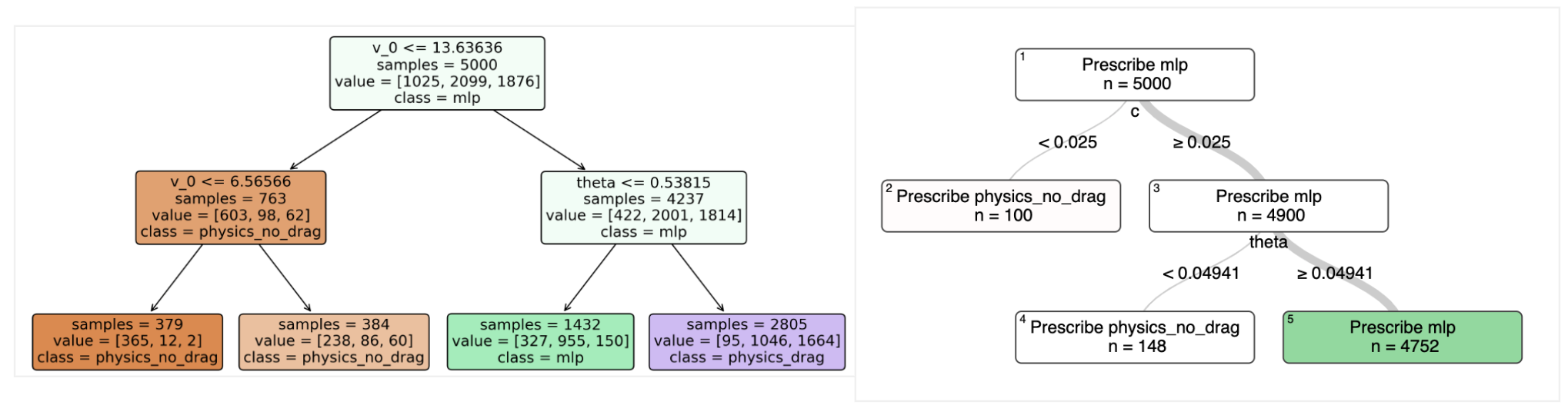}
    \caption{Left: A depth-2 Meta-Tree for the projectile motion problem. Right: A depth-2 OP\textsuperscript{2}T for the same problem. The Meta-Tree achieves an out-of-sample MSE of 79344.09, while the OP\textsuperscript{2}T achieves an MSE of 183.26. The best single model in hindsight, the MLP, achieves an MSE of 241.16. This demonstrates the benefit of prescription over prediction in this setting.}
    \label{fig:physics-depth-2-trees}
\end{figure}

\begin{table}[!h]\label{table:physics-max-depth}
    \centering
    \begin{tabular}{|c|c|c|}
    \hline
        \multirow{2}{*}{Max Depth} & \multicolumn{2}{c|}{MSE}\\
        \cline{2-3} & Meta-Tree & OP\textsuperscript{2}T\\
        \hline
        2 & $79344.09$ & 183.26\\
        3 & $317.13$ & 165.95\\
        4 & $216.21$ & 162.82\\
        5 & $171.307$ & 156.21\\
        \hline
    \end{tabular}
    \caption{Out-of-sample MSE for the Meta-Tree and OP\textsuperscript{2}T approaches for shallow tree depths on the projectile motion dataset. Such shallow structures are necessary in settings where interpretability is crucial.}
    \label{tab:projectile-tree-op2t}
\end{table}

We observe that the OP\textsuperscript{2}T approach significantly outperforms Meta-Tree, demonstrating the impact of taking a prescriptive approach. Further, by restricting the depth of the trees, we can observe an even more drastic disparity between the OP\textsuperscript{2}Ts and Meta-Trees. For a depth-2 tree with a minimum leaf size of 50, our OP\textsuperscript{2}T model achieved an MSE of 183.26, while the Meta-Tree achieved an MSE of 79344.09. A visualization of these trees is given in Figure \ref{fig:physics-depth-2-trees}. This result is due to the extremely poor accuracy of the closed-form physics equations when the conditions of the system are far from the assumptions made for them, and the Meta-Tree only being able to consider which model is best for each sample, rather than the relative reward for each model. The Meta-Tree finds partitions of the space where the physics equations perform the best for the majority of samples (minimizing entropy), but the predicted model performs orders of magnitude worse on the remaining samples in the partition. For example, the Meta-Tree predicts $\bar x_f$, the physics equation with drag, when $v_0 \geq 13.63$ and $\theta \geq 0.538$. From the data used to train the tree, $2805$ samples fall into this partition, of which a majority ($N=1664$) are best predicted by $\bar x_f$. However, the average reward (MSE) for $\bar x_f$ over the samples in this leaf node is $135197.52$, while the average reward for the MLP model is $220.38$, orders of magnitude lower. The depth-3 Meta-Tree eliminates most of this error by replacing this leaf with a split on the drag coefficient, $c$, only predicting $\bar x_f$ when $c \geq 0.295$. In contrast, we see in Figure \ref{fig:physics-depth-2-trees} that the depth-2 OP\textsuperscript{2}T first splits on $c \leq 0.025$, essentially determining if the drag force is negligible, matching the true physical assumption made to derive Eq. \eqref{eq:physics-no-drag}. We provide a comparison of the Meta-Tree and OP\textsuperscript{2}T for shallow depths in Table \ref{tab:projectile-tree-op2t}. Since trees become less interpretable with increasing depth, the performance of shallow trees is crucial for applications that require interpretability. For this dataset, we can see that the OP\textsuperscript{2}T approach yields shallow trees that are significantly more accurate than the Meta-Tree approach, providing a concise, interpretable policy for selecting whether to use a physics-based model or a pattern-based deep learning model.

\subsection{IMDb Sentiment Analysis: Glass-box vs. Black-box Models}
The IMDb Movie Review Dataset, introduced by \cite{imdb_original}, is a collection of $n=50k$ movie reviews from the IMDb website, along with binary sentiment labels, indicating positive or negative sentiment. When \cite{imdb_original} introduced this dataset, they developed and trained a Word2Vec model for sentiment prediction on this dataset. This approach preserved interpretability, allowing for both faithful explanations of model predictions as well as the ability to investigate the learned word embedding space. However, with the development of deep contextual encoders, such as BERT, introduced by \cite{devlin_bert_2019}, along with large public model repositories like HuggingFace, it is now possible to download a BERT model fine-tuned on this dataset in a matter of minutes that significantly outperforms the original Word2Vec model on this dataset across every standard metric. In the context of high-stakes decision-making problems, however, it is often the case that interpretable models are preferred. This naturally leads to the problem of how to best balance interpretability and performance, which has been well-studied and debated in the literature \citep{rudin_interpretable_2021}. In \cite{singh_augmenting_2023}, the authors develop an interpretable language model for sentiment analysis and explore the idea of adaptively selecting either their interpretable model or a black-box, BERT model. They do this by manually fixing confidence thresholds for deferring to a BERT model based on the output of the interpretable model. In their experiments, they demonstrate that performance can be preserved while using the interpretable model for a significant portion of samples across a variety of prediction tasks. However, such a process is inefficient, as it requires manually testing various thresholds, and it only considers the confidence of the interpretable model, rather than considering the confidence of the interpretable model and BERT model jointly. The problem is further complicated when there are more than two models under consideration.

In the context of natural language processing (NLP) and sentiment analysis, there are a variety of interpretable modeling approaches. In this experiment, we consider three models: a BERT model, a Word2Vec model, and a logistic regression model using term-frequency inverse-document-frequency (TF-IDF) features. The latter two models are fairly interpretable, as they are linear functions over the words contained in the reviews. All models were fit on the IMDb training dataset ($n=25k$ samples). To fit the policy models, we sampled $n=10k$ samples from the IMDb test set, and used the remaining $n=15k$ samples as our final test set. For this application, since the dataset is balanced and the primary metric is accuracy, we use the misclassification reward $R_{MIS}$, to fit our OP\textsuperscript{2}T models. Since we are interested in answering the question of interpretability, we only consider selecting the individual constituent models, not ensembles, so we set $\mathbf{W} = \text{diag}(1,1,1)$. For the feature space $\mathcal{Z}$, we consider the space of constituent model outputs, which we will also refer to as the model confidence scores. We originally considered extending this feature space with additional, structured meta-data from the reviews (e.g. review length, keywords, negation). However, we found that for this setup, the resulting OP\textsuperscript{2}Ts never used this meta-data, so we focus instead on the feature space being the model confidence scores alone. 

 The main results of our experiments are given in Table \ref{tab:imdb_results}. The best constituent model in hindsight was the BERT model, achieving an out-of-sample accuracy of 92.8. We can then observe that our OP\textsuperscript{2}T approach outperforms the best model in hindsight as well as the Meta-XGB and Meta-Tree approaches across all levels of rejection. Further, the policies learned by our method find partitions of the data where the interpretable models empirically maximize the expected reward. Therefore, with access to model outputs alone, we can construct an interpretable policy for adaptive model selection that outperforms the black-box baseline model and makes use of the interpretable models. 

\begin{table}
\resizebox{\linewidth}{!}{
    \begin{tabular}{|c|cc:cc:cc:cc|}
        \hline
        \multirow{2}{*}{Model}  & \multicolumn{2}{c:}{No Rejection} & \multicolumn{2}{c:}{Rejection ($\alpha = 0.15$)} & 
        \multicolumn{2}{c|}{Rejection ($\alpha = 0.10$)} &  \multicolumn{2}{c|}{Rejection ($\alpha = 0.02$)}\\
        \cline{2-9}
        & Acc & Reject & Acc & Reject & Acc & Reject & Acc & Reject\\
        \hline
        Meta-XGB  & 93.44 & 0 & 97.27 & 16.44 & 98.04 & 24.43 & 98.53 & 49.17 \\
        Meta-Tree  & 92.78 & 0 & 97.22 & 16.66 & 98.01 & 25.02 & 98.50 & 49.30 \\
        OP\textsuperscript{2}T  & 93.53 & 0 & 97.44 & 15.34 & 98.14 & 23.03 & 99.13 & 48.63\\
        \hline
    \end{tabular}
    }
    \caption{Out-of-sample accuracy on the IMDb Movie Review dataset across different thresholds for rejection.}
    \label{tab:imdb_results}
\end{table}

A visualization of the learned OP\textsuperscript{2}Ts is given in Figure \ref{fig:imdb-trees-None-0.15}. We see that over the space of model confidence scores, our OP\textsuperscript{2}T approach finds intersections of confidence intervals over all three of our models to provide prescriptions that outperform the Meta-Tree and Meta-XGB approaches. For example, in tree (a) in Figure \ref{fig:imdb-trees-None-0.15}, we see that the learned policy is to use the BERT model if its output score is outside the interval $(0.012, 0.992)$, otherwise, query the interpretable models. Notice that rather than simply deferring to one of the interpretable models, the policy looks at the confidence scores of both interpretable models, and will still prescribe the BERT model if the TF-IDF model is not confident the sample is positive while the Word2Vec model is not confident the sample is negative. Similarly, tree (b) contains intersections of confidence intervals between the Word2Vec and BERT models, prescribing rejection when the BERT model is not confident and in agreement with the Word2Vec model. Since many of the resulting OP\textsuperscript{2}Ts are composed of intersections of all three constituent models' confidence scores, the policies generated for model selection would not feasibly be discovered by manually searching over confidence intervals.



\begin{figure}[!h]
\centering
\subfloat[No Rejection]{\includegraphics[width = 0.4\linewidth]{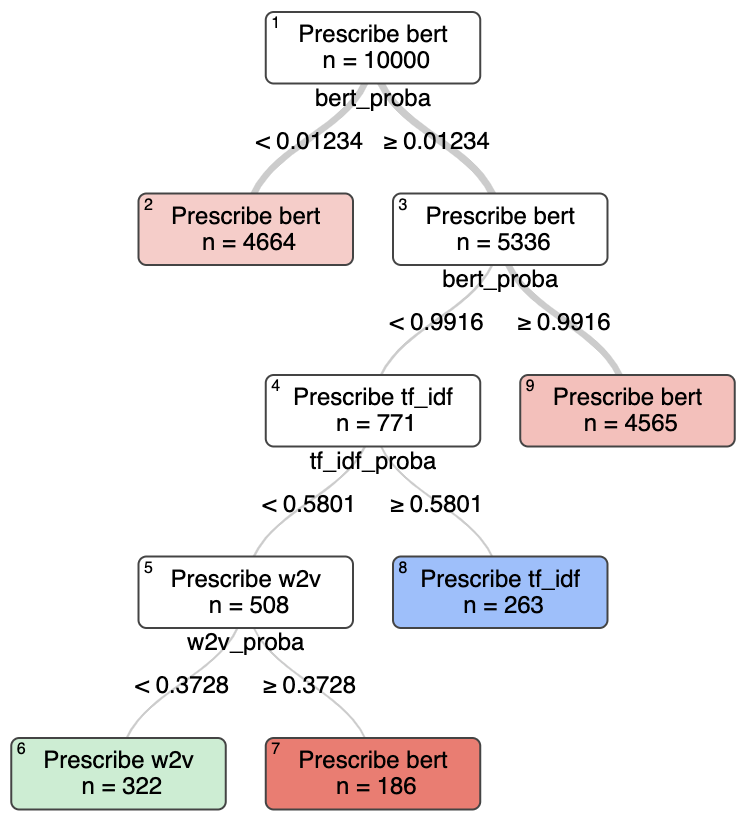}} 
\subfloat[Rejection $\alpha = 0.15$]{\includegraphics[width = 0.6\linewidth]{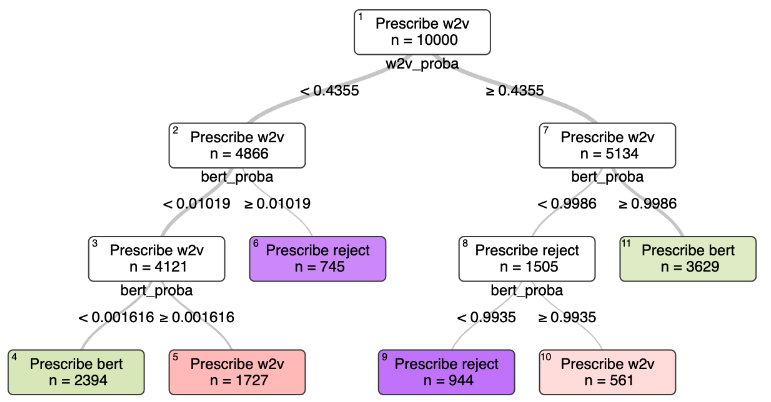}}\\
\caption{Visualizations of the OP\textsuperscript{2}Ts fit on the IMDb dataset using the original feature space extended with the constituent model scores. We observe that the best performing OP\textsuperscript{2}Ts rely exclusively on the model scores.}
\label{fig:imdb-trees-None-0.15}
\end{figure}
\begin{figure}[!h]
\centering
\subfloat[OP\textsuperscript{2}T with $\alpha = 0.02$]{\includegraphics[width = 0.5\linewidth]{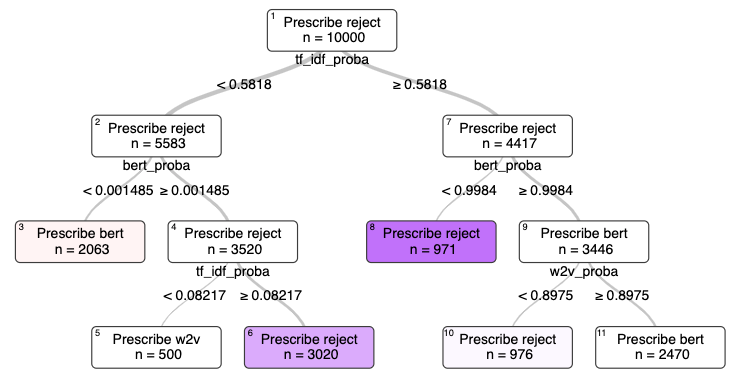}} 
\subfloat[Meta-Tree with $\alpha = 0.02$]{\includegraphics[width = 0.5\linewidth]{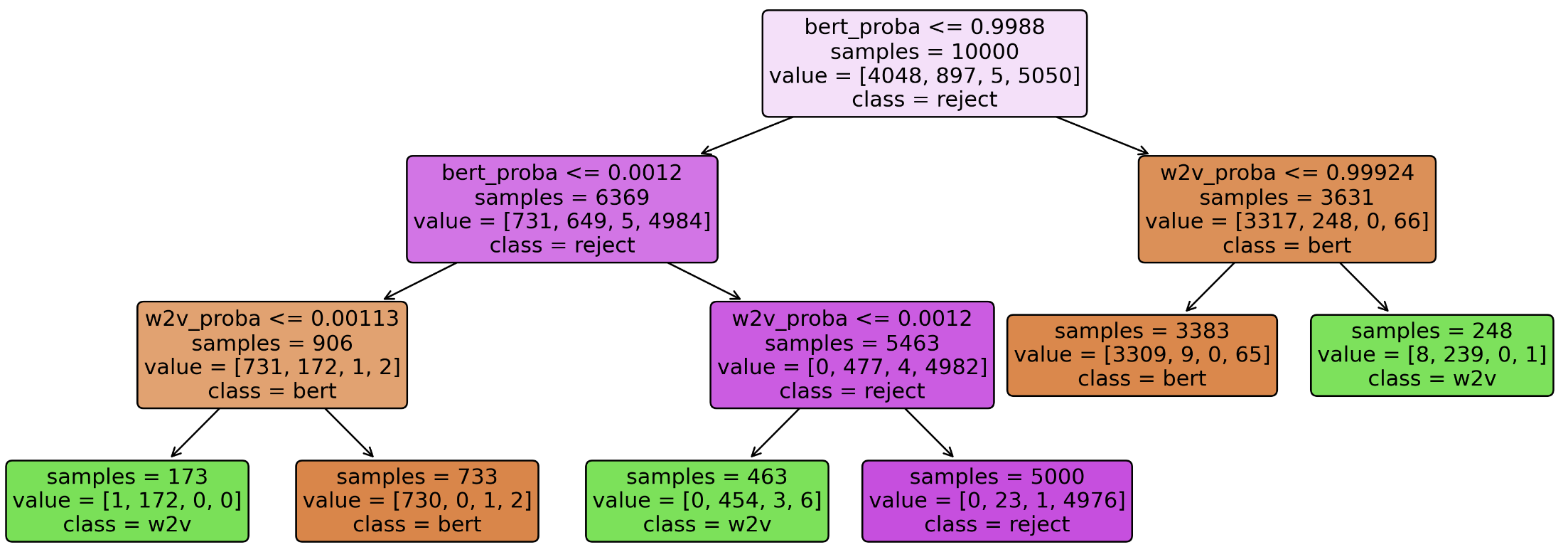}}\\
\caption{Visualization of an OP\textsuperscript{2}T and a Meta-Tree fit on the IMDb dataset with a high rejection reward ($\alpha = 0.02$).}
\label{fig:imdb-trees-0.02}
\end{figure}

In Figure \ref{fig:imdb-trees-0.02}, we give a visualization of the OP\textsuperscript{2}T and corresponding Meta-Tree generated with $\alpha = 0.02$. The OP\textsuperscript{2}T achieves higher out-of-sample accuracy while rejecting fewer samples. We provide this visualization to highlight the structural differences that can occur between Meta-Trees and OP\textsuperscript{2}Ts. By optimizing tree structure globally and considering relative rewards, the OP\textsuperscript{2}T makes use of the confidence scores of all three constituent models and identifies three separate regions within which rejection is prescribed. In contrast, the Meta-Tree, grown using a greedy heuristic and unable to consider the relative rewards of each model, finds simpler conditions for rejection, resulting in lower accuracy and more samples being rejected. Our results are also consistent with our theoretical analysis in Section \ref{sec:theory-prescription}, as we see a smaller gap between OP\textsuperscript{2}T and the alternative approaches due to the small difference on average in relative reward between the constituent models. To summarize, we can connect our results back to our motivating questions as follows.

\noindent \textbf{Should an interpretable model be used?} Yes, in many cases the interpretable models can outperform the black-box BERT model. Specifically, when the BERT model disagrees with the predictions of the interpretable models, and the BERT model is not very confident, using one of the interpretable models improves performance. In Figure \ref{fig:imdb-trees-None-0.15}, tree (b) prescribes the Word2Vec model for 23\% of samples.

\noindent \textbf{When are the models likely to be error-prone?} In short, when the models disagree, and none of the models are very confident, all of the models are likely to be error-prone. The resulting trees give specific conditions for these cases, such as in Figure \ref{fig:imdb-trees-0.02}, OP\textsuperscript{2}T (a) prescribes rejection when the output of the TF-IDF model is in the interval $[0.082, 0.582]$ and the BERT model output is in the interval $[0.0015, 1]$.



\subsection{MIMIC-IV}
The MIMIC-IV dataset is a publicly available medical dataset consisting of electronic health records (EHRs) from patients admitted to the intensive care unit at Beth Israel Deaconess Medical Center in Massachusetts \citep{johnson_mimic-iv_2023}. The dataset includes patient measurements, medications, diagnoses, treatments, and free-text clinical notes. In this section, we evaluate our methodology on the task of predicting mortality during hospital stays using only the tabular measurement data available in MIMIC-IV.

For the tabular mortality prediction task, we have a cohort of $n =11306$ separate patient visits with $14$ features representing patient measurements such as age, body temperature, glucose, and white blood cell count. There is missingness in this dataset, and we replace all missing entries with a special value of $-1$ (all measurements in the dataset are non-negative). For our constituent models, we create and train a random forest model (RF), a multi-layer perceptron (MLP) model, and a logistic regression model (LR), and two boosted tree models: an XGBoost model as in \cite{xgboost} and an Explainable Boosting Machine (EBM) model as in \cite{nori2019interpretml}. For our set of ensembles $\mathbf{W}$, we include the mean ensemble $\mathbf{w} = [\frac{1}{5},\dots, \frac{1}{5}]$, in addition to every pair-wise ensemble $\mathbf{w}_{ij} = \frac{1}{2}(\mathbf{e}_i + \mathbf{e}_j)$ for all possible pairs of constituent models. We split the data $60/30/10$ as training, validation, and test datasets. We fit our OP\textsuperscript{2}Ts and benchmark models on the validation set. Further, for fitting the policy models, we experiment with two versions of the feature space: the original feature space, and the original feature space extended by the output scores of the constituent models. We evaluate the performance of these methods with ROC-AUC on the test dataset.

\begin{table}
    \centering
    \begin{tabular}{|c|cc:cc:cc|}
        \hline
        \multirow{2}{*}{Model} & \multicolumn{2}{c:}{No Rejection} & \multicolumn{2}{c:}{Rejection w/ $\alpha = 0.45$} & \multicolumn{2}{c|}{Rejection w/ $\alpha = 0.40$}\\
        \cline{2-7}
         & AUC & Reject & AUC & Reject & AUC & Reject \\
        \hline
        Meta-XGB & 0.9224 & 0 & 0.9468 & 18.57 & 0.9639 & 32.98  \\
        Meta-Tree & 0.9137 & 0 & 0.9419 & 19.36 & 0.9609 & 35.46\\
        OP\textsuperscript{2}T  & 0.9235 & 0 & 0.9478 & 18.83 & 0.9673 & 32.97 \\
        \hline
    \end{tabular}
    \caption{Out-of-sample Accuracy on the MIMIC-IV dataset across different thresholds for rejection. The feature space for the meta-learners is the union of the original feature space and the constituent model scores (not including ensembles).}
    \label{tab:mimic_results_scores}
\end{table}

\begin{table}
    \centering
    \begin{tabular}{|c|cc:cc:cc|}
        \hline
        \multirow{2}{*}{Model} & \multicolumn{2}{c:}{No Rejection} & \multicolumn{2}{c:}{Rejection w/ $\alpha = 0.40$} & \multicolumn{2}{c|}{Rejection w/ $\alpha = 0.35$}\\
        \cline{2-7}
         & AUC & Reject & AUC & Reject & AUC & Reject \\
        \hline
        Meta-XGB & 0.9209 & 0 & 0.9455 & 22.19 & 0.9718 & 43.41  \\
        Meta-Tree & 0.9152 & 0 & 0.9422 & 23.52 & 0.9675 & 42.00\\
        OP\textsuperscript{2}T  & 0.9235 & 0 & 0.9476 & 21.84 & 0.9736 & 42.23 \\
        \hline
    \end{tabular}
    \caption{Out-of-sample Accuracy on the MIMIC-IV dataset across different thresholds for rejection. The feature space for the meta-learners is the original feature space.}
    \label{tab:mimic_results_features}
\end{table}

 The results for this experiment are summarized in Tables \ref{tab:mimic_results_scores} and \ref{tab:mimic_results_features}. We observe that the OP\textsuperscript{2}Ts strictly dominate the Meta-Tree and Meta-XGB methods on this dataset, and that including the constituent model scores further improves the out-of-sample performance of the  OP\textsuperscript{2}Ts, increasing ROC-AUC while decreasing the percent of samples rejected. In Figure \ref{fig:mimic-viz-scores}, we give a visualization of the OP\textsuperscript{2}Ts that were fit on the extended feature space, including the constituent model scores. In this case, we observe that the model scores are the most important features and are used exclusively in the best performing OP\textsuperscript{2}Ts. We also observe stability in the tree structure, as trees (a) and (b) share the same first split and a similar overall structure and model prescriptions. In fact, the trees share the same structure, with the second and third split flipped, and the thresholds for both splits being increased as the reward for rejection is increased.

\begin{figure}[h!]
\subfloat[Rejection $\alpha = 0.45$]{\includegraphics[width = 0.55\linewidth]{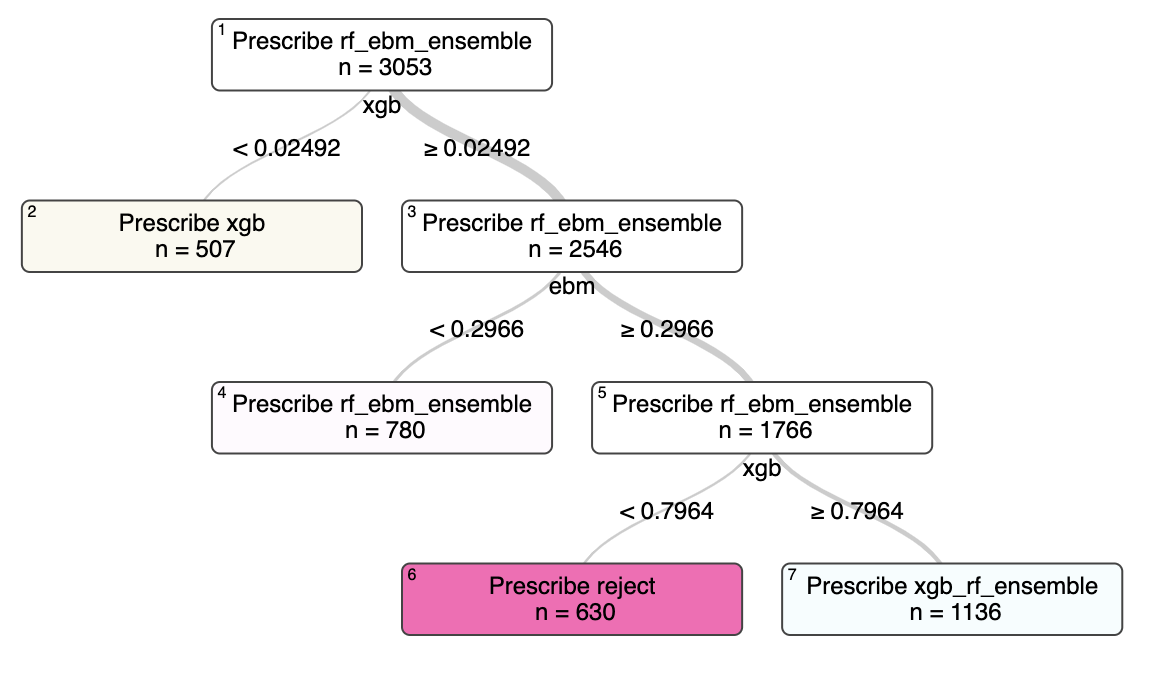}} 
\subfloat[Rejection $\alpha = 0.40$]{\includegraphics[width = 0.45\linewidth]{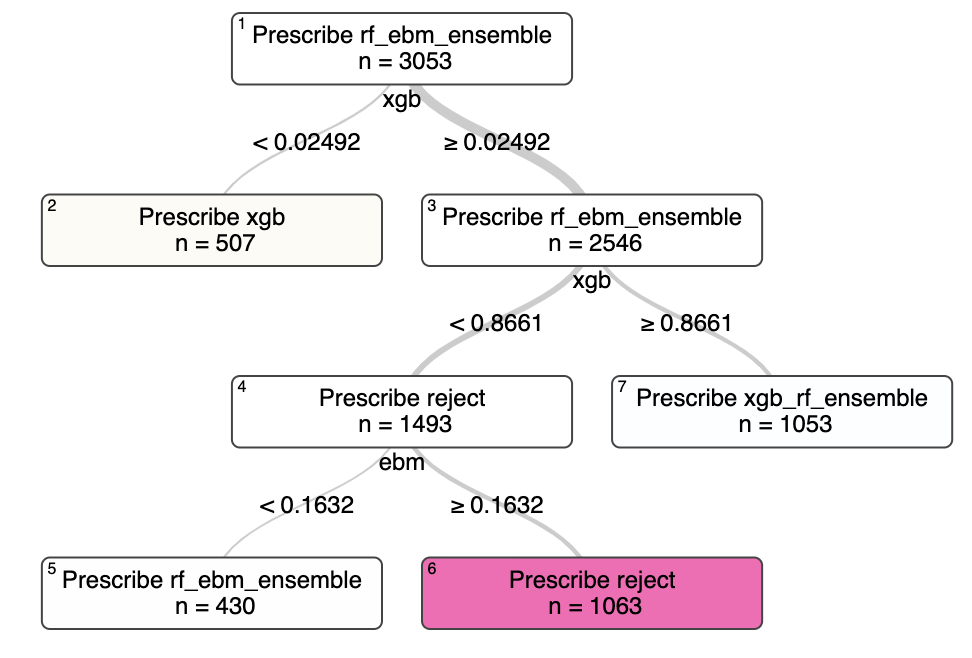}}\\
\caption{Visualizations of the OP\textsuperscript{2}Ts fit on the MIMIC-IV dataset using the original feature space extended with the constituent model scores. We observe that the best performing OP\textsuperscript{2}Ts rely exclusively on the model scores.}
\label{fig:mimic-viz-scores}
\end{figure}

\begin{figure}[h!]
    \centering
    \includegraphics[width=1\linewidth]{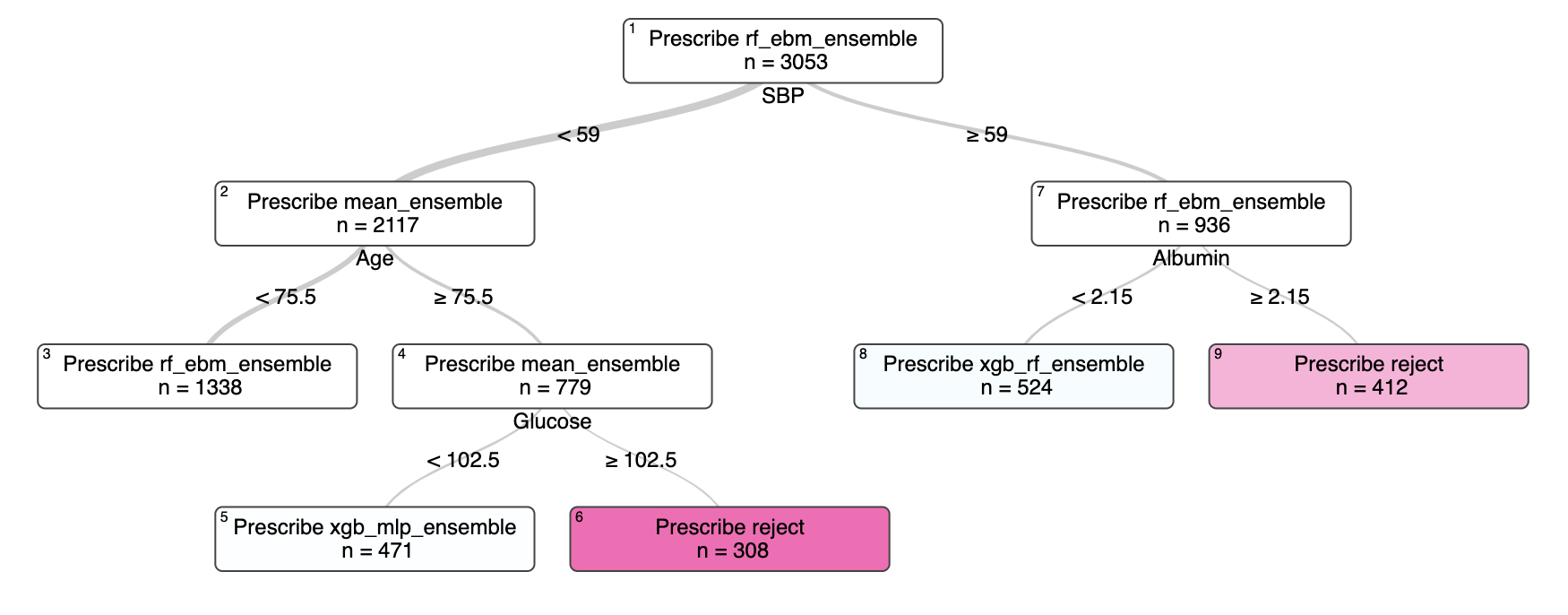}
    \caption{The OP\textsuperscript{2}T fit on the MIMIC-IV dataset using the original feature space with rejection parameter $\alpha = 0.4$.}
    \label{fig:mimic-feat-viz-0.4}
\end{figure}

\begin{figure}[h!]
    \centering
    \includegraphics[width=0.75\linewidth]{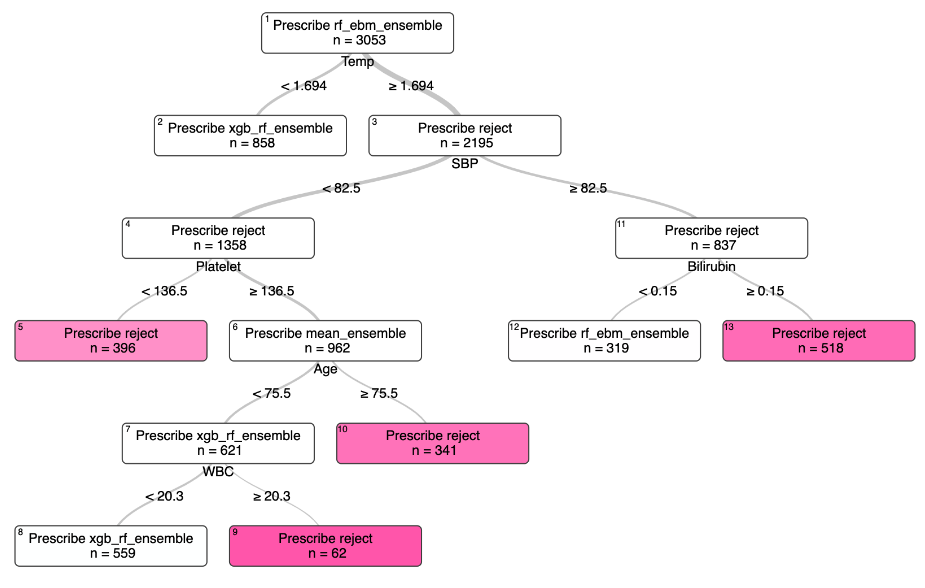}
    \caption{The OP\textsuperscript{2}T fit on the MIMIC-IV dataset using the original feature space with rejection parameter $\alpha = 0.35$.}
    \label{fig:mimic-feat-viz-0.35}
\end{figure}

In Figures \ref{fig:mimic-feat-viz-0.4} and \ref{fig:mimic-feat-viz-0.35} we give visualizations of the OP\textsuperscript{2}Ts fit on the original feature space. While we see both trees share splits on systolic blood pressure (SBP) and age, the OP\textsuperscript{2}T in Figure \ref{fig:mimic-feat-viz-0.35} with the higher reward for rejection uncovers additional conditions under which all models perform poorly. Further, the conditions for rejection have some clinical meaning. For example, in Figure \ref{fig:mimic-feat-viz-0.4}, we see a split on glucose at 102.5 $mg/dL$, which is consistent with the upper bound for the normal range, 100 $mg/dL$, for glucose as provided by the World Health \cite{WHO}. Taken together, we can interpret one set of conditions for rejection as elderly patients with low blood pressure and high glucose levels, potentially diabetic or pre-diabetic. In Figure \ref{fig:mimic-feat-viz-0.35}, we observe other conditions that are important for rejection, such as low platelet count (normal range is 150 to 450) or high white blood cell (WBC) count (normal range is 4.5 to 11). Therefore, we see that OP\textsuperscript{2}Ts not only provide a performance advantage over similar methods but they can also uncover sub-groups within the data on which all constituent models, including the ensembles, struggle.

\section{Conclusion}
In this work, we introduce a prescriptive framework for creating interpretable and adaptive model selection and ensembling policies, along with a parameterized rejection option. We demonstrate on a variety of tasks, model classes, and datasets, that our approach yields strong performance while offering interpretable policies that aid in the understanding of the efficacy of the constituent models. In particular, our OP\textsuperscript{2}T approach directly provides answers to the questions we identified in the introduction, such as under what conditions different models should be used, and when all of the models are likely to be error-prone. In contrast, traditional ME approaches fail to answer these questions, instead making use of all the models at all times, weighting them using a method that is typically a black-box itself. Our approach also makes minimal assumptions regarding the constituent models, only requiring access to model outputs, increasing the applicability of our method. In addition to empirical results, we introduce a theoretical framework for adaptive model selection with policies learned over side information, only assuming access to model outputs. We show both theoretically and empirically that our prescriptive approach can vastly outperform predictive methods when the data is not separable. We further demonstrate empirically that the resulting trees created using our method are stable, and with a rejection option, tend to converge to a particular structure as the reward for rejection is increased. We provide empirical evidence that this work could be particularly beneficial in settings where the constituent models come from significantly different model classes, such as physics-based models and pattern-based, ML models such as boosted trees and neural networks. There are many exciting directions for future work, such as investigating different reward functions, dynamically adapting ensemble weights, and exploring the theoretical properties of such expert selection systems. This work takes a first step toward using prescriptive analytics to enhance predictive models, and we hope it encourages further research in this direction.


\acks{The authors acknowledge the MIT SuperCloud and Lincoln Laboratory Supercomputing Center for providing high-performance computing resources that have contributed to the research results reported in this paper.}


\newpage

\appendix
\section{Convergence of OP\textsuperscript{2}Ts for MIMIC-IV}\label{sec:converge}
In Figure \ref{fig:mimic-converge}, we can observe how our OP\textsuperscript{2}T models for the MIMIC-IV dataset vary slowly for decreasingly small values of $\alpha$, identifying a consistent set of features and splits. For example, decreasing the rejection parameter from $\alpha = 0.25$ to $\alpha = 0.23$, corresponding to trees (a) and (b), we see that the splits at the root nodes are almost identical, and the split in the right-subtree is exactly the same. The only difference is the left sub-tree, where the split in tree (a) has essentially been pruned in tree (b). Similarly, decreasing the rejection parameter from $\alpha = 0.18$ to $\alpha = 0.15$, corresponding to trees (c) and (d), we see that the same split is made on the confidence of the XGBoost model and, subsequently, a similar split is made on the confidence of the random forest model. The only difference is that tree (d) further splits in the left sub-tree on white blood cell count, identifying another region of the feature space where all constituent models are error-prone. In this case, the region identified corresponds to patients for whom the XGBoost model is confident will survive but who have low WBC counts, which increases the likelihood of developing infections. This could suggest that the XGBoost model can be overconfident for patients who are otherwise low-risk but have a low WBC count, which could potentially lead to complications during their stay in the hospital.

\begin{figure}[h!]
    \centering
    \includegraphics[width=1\linewidth]{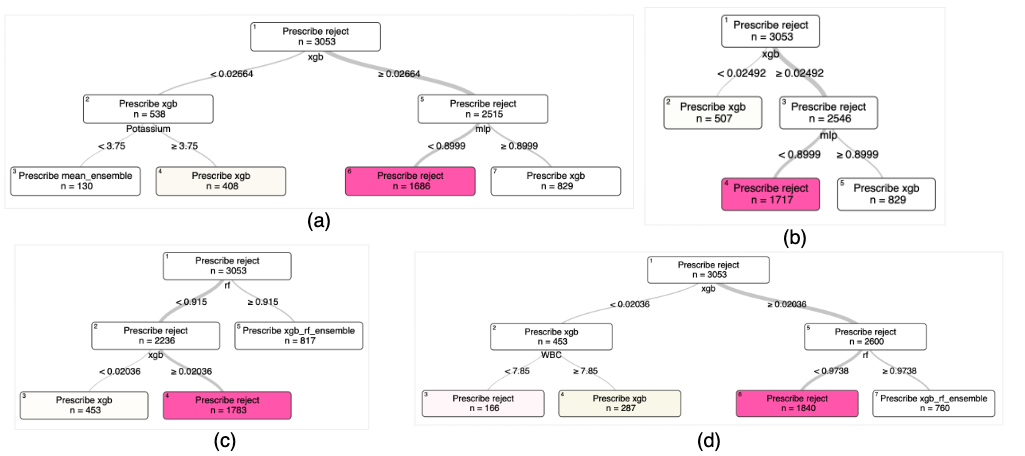}
    \caption{Convergence of OP\textsuperscript{2}Ts on the MIMIC-IV dataset for low rejection thresholds. Trees (a)-(d) correspond to rejection parameters $\alpha = [0.25, 0.23, 0.18, 0.15]$ respectively. The trees vary slowly and contain many similar or identical splits.}
    \label{fig:mimic-converge}
\end{figure}

For some leaf nodes, such as the leaf node in Tree (c) that prescribes the XGBoost model, the conditions for a sample falling into the leaf imply that the XGBoost model is near certain that these samples come from the negative class. Such a leaf node is equivalent to predicting the negative class directly for all prediction thresholds greater than $0.02036$. This also illustrates the connection between our approach, and the related formulation discussed in Appendix \ref{apdx:rej-intervals}.

\section{Learning Multi-Model Rejection Intervals for Classification}\label{apdx:rej-intervals}
In the classification setting where we have access to model scores, but not necessarily an interpretable or structured feature space $\mathcal{X}$ or side-information $\mathcal{Z}$, a natural extension of our work is to learn a policy that directly prescribes a class prediction, or rejects the input sample. In this case, we can learn a policy over the constituent model outputs alone and do not consider the input features. We first introduce the single-model setting, and then the multi-model setting, where we can apply our framework for policy learning. 
\subsection{The Single-Model Case}
We begin by studying the simplest case, binary classification with a single classifier. Specifically, we assume $\mathcal{C} = \{0, 1\}$ and we have some classifier $h: \mathcal{X} \rightarrow [0, 1]$, where the output $h(x)$ represents the likelihood of sample $x$ belonging to class $C = 1$. For now, let us assume that we just want to set a single threshold $\alpha$ for the accuracy of the classifier on the samples for which we decide to make a prediction. We will show how we can extend our formulation to include additional thresholds on other metrics, such as the false negative rate (FNR) or false positive rate (FPR). We are interested in learning an \textbf{interval for rejection}, $[a, b]$, for some $a, b \in [0, 1]$ such that $a \leq b$. If $h(x) \in [a, b]$, we reject, otherwise we predict $C=1$ if $h(x) \geq b$ or $C=0$ if $h(x) \leq a$. We are able to restrict ourselves to the set of intervals $[a, b] \subset [0, 1]$ because we can safely assume $P(y = 1 | h(x))$ is increasing monotonically in $h(x)$.

In addition to finding a rejection interval  $[a, b]$ that satisfies the accuracy constraint, we would like to make predictions as often as possible. Otherwise, we will learn to reject all the time (that is, $a = 0$ and $b =1$). Our goal is then to maximize the probability that the output of our classifier falls outside the rejection interval, 
$P(h(x) \notin [a, b])$, such that the accuracy constraint, $P(y = \hat y | h(x) \notin [a,b]) \geq \alpha$, is satisfied, where $\hat y = \mathbbm{1}_{h(x) \geq b}$ is the predicted class. Since we only have access to finite data, $\{(x_i, y_i)\}_{i=1}^n$, 
we replace $P(h(x) \notin [a, b])$ with $\frac{|I(a,b)|}{n}$ and $P(y = \hat y | h(x) \notin [a,b])$ with $\frac{1}{|I(a,b)|}\sum_{i \in I(a,b)} \mathbbm{1}_{y_i = \hat y_i}$, where $I(a,b) = \{ i \in [n] : h(x_i) \geq b \vee h(x_i) \leq a \}$ is the set of data points for which we make a prediction for a given rejection interval $[a, b]$.

This formulation can be written as a mixed integer optimization (MIO) problem. We define the problem as follows, beginning with the decision variables:
\begin{itemize}
    \item $a, b \in [0, 1]$ endpoints of rejection interval.
    \item $\mathbf{z} \in \{0, 1\}^n$ binary indicator variables, $z_i = 1$ if $x_i$ is outside of interval $[a, b]$ ($x_i \in I(a,b)$), otherwise $z_i = 0$.
    \item $\mathbf{\hat y} \in \{0, 1\}^n$ binary variables representing the predictions made by $h$ given the rejection interval $[a, b]$
\end{itemize}
With these decision variables, we define the MIO problem as follows
\begin{equation}\label{eq:single-model}
    \begin{aligned}
        \max_{a,b, \mathbf{\hat y}, \mathbf{z}} \quad & \sum_{i=1}^n z_i \\
        \text{s.t.} \quad & z_i \geq h(x_i) - b \quad \forall i \in [n], \\
        \quad & z_i \geq a - h(x_i) \quad \forall i \in [n], \\
        \quad & z_i \leq 1  - (h(x_i) - a)(b - h(x_i)) \quad \forall i \in [n], \\
        \quad & \hat y_i \geq h(x_i) - b \quad \forall i \in [n], \\
        \quad & \hat y_i \leq h(x_i) - a + 1 \quad \forall i \in [n], \\
        \quad & \sum_{i=1}^n z_i(y_i\hat y_i + (1 - y_i)(1 - \hat y_i)) \geq \alpha \sum_{i=1}^n z_i,\\
        \quad & a \leq b, \quad a,b \in [0,1].\\ 
    \end{aligned}
\end{equation}

The first five constraints in Formulation \eqref{eq:single-model} construct the indicator variables $z$ and $\hat y$ as we have defined them above. The sixth constraint is the accuracy constraint over $I(a, b)$ at a level $\alpha$. We allow $a \leq b$, rather than strict inequality, to allow for a single threshold solution if it is possible to classify all the data with accuracy $\alpha$ without rejecting any predictions. Note that with $(\mathbf{z}, \mathbf{\hat y}, y)$ we can compute any accuracy metric (FNR, FPR, Precision, Recall) over the data in $I(a,b)$. We can therefore add a threshold constraint for any of these metrics to the formulation (e.g. FNR $\leq \alpha_{FNR})$. Further, note that we can swap the objective and the accuracy constraint, and instead maximize accuracy subject to a constraint on the percent of samples predicted. This is a separate but closely related formulation.

This exercise serves purely to demonstrate that the single model, binary classification case can be cast as an MIO problem and tractably solved to optimality. In practice, one could adopt an iterative method that finds the optimal solution in time polynomial in the number of points in the validation set. In this approach, we iterate over all possible intervals, either by iterating over all unique values $h(x_i)$, or a pre-defined grid. The latter is a heuristic approach, but very quickly yields an approximate solution.

\subsection{Multi-Model Rejection.} Now suppose we have a collection of $m$ models $\mathbf{h} = (h_1,\dots, h_m), h_i: \mathcal{X} \rightarrow [0,1]$ that have been fit on some training dataset. Focusing on the case of binary classification, for some input $x \in \mathcal{X}$, we have output scores $\mathbf{h}(x) = (h_1(x),\dots, h_m(x)) \in [0,1]^m$ from the $m$ models. We would like to learn a function $f: [0,1]^m \rightarrow \{0,1, \text{Reject}\}$ that decides what class to predict, or if we should reject, based on the outputs $\mathbf{h}$ of the collection of models. Unlike the single model case, there are many classes of models that may be appropriate for learning this function $f$. We can adapt our OP\textsuperscript{2}T approach to learn such a function. For a given validation set $X$, we can find the output scores $\mathbf{h}(X) =\{(h_1(x_i),\dots, h_m(x_i))\}_{i=1}^n \in [0,1]^{n \times m}$. We assume our state space to be $[0,1]^m$ and $\mathbf{h}(X)$ represents samples from this state space. We define our action space as $a = \{0, 1, \text{Reject}\}$, representing either predicting the positive or negative class or rejecting the sample. Note that we can generalize this to the multi-class setting by extending the action space and state space. Given the true class labels $y$ for the validation dataset, we then define the reward for each action $a_i$ on sample $(x_i, y_i)$ as follows
\begin{equation}
    R(a_i ; x_i, y_i) = \begin{cases} 
    1 & a_i = y_i \\
    0 & a_i \neq y_i \\
    \alpha & a_i = \text{reject}
    \end{cases}
\end{equation}
Where $\alpha > 0$ is our rejection threshold. We then fit an optimal policy tree $T_O$ to the data $\{(\mathbf{h}(x_i), R(a,x_i,y_i))\}_{i=1}^n$ to learn a policy for prediction and rejection. Note that the reward matrix $R$ is complete, so there is no need for any reward estimation in this setting. Since OPTs select the policy that maximizes the total reward in each leaf, we have defined our reward so that a given OPT will only prescribe a class label in a leaf if the accuracy on the validation data in that leaf is greater than $\alpha$. 

\hfill
\subsection{Multi-Model Rejection with Imbalanced Data and Costs}
In the case that we are only concerned about accuracy, and do not assign different costs for false positives versus false negatives, the method we have described in the previous subsection is sufficient. If there are different costs associated with false positives and false negatives, or if the data is imbalanced, setting a single threshold for rejection may not be enough. In this case, we can create a collection of optimal policy trees $T_1,\dots, T_k$ that each satisfies a certain rejection threshold. For each tree $T_k$, we define a separate reward matrix
\begin{equation}
    R_k(a_i ; x_i, y_i) = \begin{cases} 
    \beta_{ky_i} & a_i = y_i \\
    0 & a_i \neq y_i \\
    \alpha_k & a_i = \text{defer}
    \end{cases}
\end{equation}
which is now parameterized by both a rejection threshold $\alpha_k$ as well as a set of weighted rewards $\beta_{k1}, \cdots, \beta_{kC}$ for each class label. With these parameters, we can encode any constraint on the accuracy for each class. For example, suppose we have an imbalanced binary classification dataset with $80\%$ of the data from the negative class and $20\%$ from the positive class. Further, suppose we care much more about false negatives than we do about false positives, so we are okay with predicting the positive class any time we expect the probability of a sample being positive to be at least $40\%$ (that is, $P(\text{positive} | \text{predicted positive}) \geq 0.4$). Then we could set $\beta_{k1} = 1$, $\beta_{k2} = 2$, and $\alpha_1 = \alpha_2 = 0.8$, and we will fit a tree $T_k$ such that for any leaf in $T_k$, the prediction in that leaf can be the positive class if at least $40\%$ of the validation data that falls into that leaf is from the positive class. We may also want to create a separate tree with a rejection threshold specific to the negative class, for example. Once we have created these trees, we can then take an action by looking at the intersection of the actions prescribed by the individual trees. In our example, we could define the following mapping:
\begin{equation}
    \begin{cases}
        \text{Prescribe Positive,} & T_1(x) = 1 \wedge (T_2(x) = 1 \vee T_2(x) = \text{Reject})\\
        \text{Prescribe Negative,} & T_2(x) = 0 \wedge (T_1(x) = 0 \vee T_1(x) = \text{Reject})\\
        \text{Prescribe Reject,} & T_1(x) = T_2(x) = \text{Reject} \vee (T_1(x),T_2(x) \neq \text{Reject} \wedge T_1(x) \neq T_2(x))\\
    \end{cases}
\end{equation}
We note that in the last case, we prescribe ``Reject'' if both trees prescribe ``Reject'' \textbf{or} if both trees prescribe a class label, but they disagree on their predictions.

\section{Proofs of Propositions}
\subsection{Rejection Learning with $R_{CE}$ Reward}\label{proof:class-rej-properties}
\begin{proposition}\label{prop:class-rej-properties}
Suppose we are given an OP\textsuperscript{2}T $T_{O}$ fit on some input data $\{x_i\}_{i=1}^n \in \mathcal{X}^n$, label data $\{y_i\}_{i=1}^n \in \mathcal{C}^n$, and constituent models $h_1,\dots, h_m$ with the cross-entropy reward $R_{CE}$. Further, suppose we have defined a rejection model $h_r$ parameterized by constants $\boldsymbol{\alpha} = (\alpha_1,\dots, \alpha_K)$. For any leaf node $l \in T_{O}$, we denote by $i \in l$ the case that input sample $x_i$ falls into leaf $l$. Then for any leaf node $l$ with $|l| = n_l$, the tree $T_{O}$ prescribes the rejection model $h_r$ if and only if
$$\sum_{i \in l} \log \left(\frac{1-\alpha_{y_i}}{\hat y_{ijy_i}}\right) > 0 \quad \forall j \in [m].$$
Further, if $\alpha_i = \alpha$ for all $i \in [K]$, the critical rejection threshold is given by
$$\alpha^* = 1 - \max_{j \in [m]} \exp\bigg(\frac{1}{n_l}\sum_{i \in l} \log(\hat y_{ijy_i})\bigg),$$
such that for all $\alpha < \alpha^*$ the leaf prescription will be to reject.
\end{proposition}
\begin{proof}
    For any fixed leaf node $l$, the total reward for rejection is $\sum_{i \in l} \log(1-\alpha_{y_i})$ and for any constituent model $h_j$ the total reward is $\sum_{i \in l}\log(\hat y_{ijy_i})$. Since the leaf prescribes the model with the greatest total reward, it follows that rejection prescribes if and only if 
    $$\sum_{i \in l} \log(1-\alpha_{y_i}) > \sum_{i \in l}\log(\hat y_{ijy_i}).$$
    Rearranging terms, we arrive at the expression 
    $$\sum_{i \in l}\log\bigg(\frac{1-\alpha_{y_i}}{\hat y_{ijy_i}}\bigg) > 0,$$
    as in the proposition. Now, if we assume $\boldsymbol{\alpha}$ is constant such that $\alpha_i = \alpha$ for all $i \in [K]$, the total reward for rejection in leaf $l$ becomes $n_l\log(1-\alpha)$. Then we can find the critical rejection threshold by setting the reward for rejection equal to the maximum total reward over all the constituent models as follows
    \begin{align*}
        n_l\log(1-\alpha) & = \max_{j \in [m]}\sum_{i \in l}\log(\hat y_{ijy_i}) \\
        1 - \alpha & = \max_{j \in [m]}\exp\bigg(\frac{1}{n_l}\sum_{i \in l}\log(\hat y_{ijy_i})\bigg) \\
        \alpha^* & = 1 - \max_{j \in [m]}\exp\bigg(\frac{1}{n_l}\sum_{i \in l}\log(\hat y_{ijy_i})\bigg).
    \end{align*}
    We can pull the max term out of the exponent because $e^x$ is a monotonic function. Finally, note that $\sum_{i \in l} \log\big(\frac{1-\alpha}{\hat y_{ijy_i}}\big)$ is strictly monotonic decreasing in $\alpha \in (0, 1)$. We can therefore conclude that for all $\alpha < \alpha^*$, the leaf prescription will be to reject.
\end{proof}
\subsection{Proof of Proposition \ref{prop:rewards}}\label{proof:rewards}
\begin{proof}
We prove this proposition by construction. We first define the model that maximizes the expected reward over all of $\mathcal{X} \sim \mathcal{D}$:
\begin{equation}
    h^* = \argmax_{h \in H} E_{\mathcal{D}}[R(x, f^*(x), h)]
\end{equation}
Let $C = \mathcal{Z} \setminus (A \cup B)$. There are two cases to consider. The first is when $C = \emptyset$ or $P_D(\cup_{z \in C}v(z)) = 0$. In this case, we can define the following policy function:
\begin{equation}
    \pi(z) = \begin{cases}
        h_A^* & z \in A\\
        h_B^* & z \in B\\
    \end{cases}
\end{equation}
When $C \neq \emptyset$ and $P_D(\cup_{z \in C}v(z)) > 0$, we can define
\begin{equation}
    h_C^* = \argmax_{h \in H} E_{\mathcal{D}}[R(x, f^*(x), h) \cdot \mathbbm{1}\{x \in g(C)\}],
\end{equation}
such that 
$$E_{\mathcal{D}}[R(x, f^*(x), h_C^*) \cdot \mathbbm{1}\{g(x) \in C\}] \geq E_{\mathcal{D}}[R(x, f^*(x), h^*) \cdot \mathbbm{1}\{g(x) \in C\}].$$
This model $h_C^*$ does not need to be unique, and it can be that $h_C^* = h_A^*$ or $h_C^* = h_B^*$. Further, $C$ may not be a dominated subspace, as it may not be connected. We can then define the following policy function:
\begin{equation}
    \pi(z) = \begin{cases}
        h_A^* & z \in A\\
        h_B^* & z \in B\\
        h_C^* & z \in C\\
    \end{cases}
\end{equation}
Since $h_A^* \neq h_B^*$, it must be that either $h^* \neq h_A^*$ or $h^* \neq h_B^*$, or both. This implies that we have either
$$E_{\mathcal{D}}[R(x, f^*(x), h_A^*) \cdot \mathbbm{1}\{g(x) \in A\}] > E_{\mathcal{D}}[R(x, f^*(x), h^*) \cdot \mathbbm{1}\{g(x) \in A\}]$$
or 
$$E_{\mathcal{D}}[R(x, f^*(x), h_B^*) \cdot \mathbbm{1}\{g(x) \in B\}] > E_{\mathcal{D}}[R(x, f^*(x), h^*) \cdot \mathbbm{1}\{g(x) \in B\}]$$
or both. Putting these inequalities together, we get
\begin{align}  
E_{\mathcal{D}}[R(x, f^*(x), \pi(g(x))] = &
\; E_{\mathcal{D}}[R(x, f^*(x), h^*_A)\cdot \mathbbm{1} \{g(x) \in A\}]  \nonumber \\
&
\; + E_{\mathcal{D}}[R(x, f^*(x), h^*_B)\cdot \mathbbm{1} \{g(x) \in B\}]  \nonumber \\
& \; +  E_{\mathcal{D}}[R(x, f^*(x), h^*_C)\cdot \mathbbm{1} \{g(x) \in C\}]   \nonumber\\
> & \; E_{\mathcal{D}}[R(x, f^*(x), h^*)].  \label{eq-dominated-ineq}
\end{align}
By construction, since we assume $A$ and $B$ are disjoint, we know $A,B,C$ are disjoint sets. Let us define $A' = \cup_{z \in A} v(z)$, $B' = \cup_{z \in B} v(z)$, and $C' = \cup_{z \in C} v(z)$. Since we assume the function $v$ maps to disjoint subsets of $\mathcal{X}$, we know $A',B',C'$ are disjoint. It follows that the sets $\{x \in \mathcal{X} : g(x) \in A\}$, $\{x \in \mathcal{X} : g(x) \in B\}$, and $\{x \in \mathcal{X} : g(x) \in C\}$ are disjoint. Further, since $g$ is defined over all $x \in \mathcal{X}$ and $A\cup B \cup C = \mathcal{Z}$, it follows that $P_D(\{x \in \mathcal{X} : g(x) \in A\} \cup \{x \in \mathcal{X} : g(x) \in B\} \cup \{x \in \mathcal{X} : g(x) \in C\}) = 1$. We can therefore split the reward function as in Eq. \eqref{eq-dominated-ineq} and by linearity of expectation, we get our result.
\end{proof}
\subsection{Proof of Proposition \ref{prop-op2t-vs-meta-tree}}
\label{appx-op2t-vs-meta-tree}
\begin{proof}
For simplicity, let us consider the case with two constituent models $\mathcal{H} = \{h_1, h_2\}$, and assume $\mathcal{X} = \mathcal{Z}$. Suppose for this set of models $\mathcal{H}$ that there exists a subset of the data $A \subseteq \{(x_i, q_i)\}_{i=1}^n$ that is not separable, either at all, or by parallel split decision trees up to some maximum depth $D_{max}$. That is, for all decision trees up to $D_{max}$, the misclassification rate (or entropy) is greater than or equal to the misclassification rate or entropy of $A$ itself. Since the data in $A$ is not separable, the best the Meta-Tree approach can do is prescribe a single constituent model for the data in $A$. Let us define $A_1 = \{ (x_i, q_i) \in A : q_i = h_1\}$ and similarly $A_2 = \{ (x_i, q_i) \in A : q_i = h_2\}$. Suppose $|A_1| > |A_2|$, such that $h_1$ corresponds to the majority class. It follows that the best a Meta-Tree can do over the set $A$ is prescribe the model $h_1$, as this minimizes the misclassification error. Now, let us define 
    $$\sum_{x_i \in A_1} R(x_i, y_i, h_1) = R_{11},$$
    $$\sum_{x_i \in A_1} R(x_i, y_i, h_2) = R_{12},$$
    $$\sum_{x_i \in A_2} R(x_i, y_i, h_1) = R_{21},$$
    $$\sum_{x_i \in A_2} R(x_i, y_i, h_2) = R_{22}.$$
Then the total reward for prescribing model $h_2$ is $R_{12} + R_{22}$, and the total reward for prescribing model $h_1$ is $R_{11} + R_{21}$. Let $\Delta_1 = R_{11} - R_{12}$, and note that $\Delta_1 \geq 0$ since $R(x_i, y_i, h_1) \geq R(x_i, y_i, h_2)$ for all $i \in A_1$ by definition. Similarly, we can define $\Delta_2 = R_{22} - R_{21}$. We assume $R$ has either no finite upper or lower bound, such that $\Delta_2$ can be arbitrarily large, as we can either send $R_{22} \rightarrow \infty$ or $R_{21} \rightarrow - \infty$. Then the difference in total reward,
$$R_{12} + R_{22} - R_{11} - R_{21} = \Delta_2 - \Delta_1,$$
can be arbitrarily large as $\Delta_2 \rightarrow \infty$. Since the OP\textsuperscript{2}T approach comes from the same hypothesis class, the best an OP\textsuperscript{2}T can do is prescribe a single model over $A$. However, the model that maximizes total reward over $A$ will be prescribed. Then in our example, since the total reward for $h_2$ will be greater than $h_1$ over $A$ as $\Delta_2 \rightarrow \infty$, an OP\textsuperscript{2}T will prescribe model $h_2$. Therefore, the total reward for the OP\textsuperscript{2}T approach can be arbitrarily greater than the reward from the Meta-Tree approach. Next, let us consider a fixed dataset of size $n$ with $m$ constituent models, and define $R_{max} = \max_{i \in [n], (j,k) \in [m]} |R(x_i, y_i, h_j) - R(x_i, y_i, h_k)|$. Using our previous construction, we know for all $i \in A_2$ that $R(x_i, y_i, h_2) - R(x_i, y_i, h_1) \leq R_{max}$. We also observe that $|A_2| < \frac{n}{2}$ since we assume $|A_1| > |A_2|$. It follows that $\Delta_2 < \frac{n}{2}R_{max}$ and, noting that $\Delta_1 \geq 0$, we have
$$\Delta_2 - \Delta_1 < \frac{n}{2}R_{max}.$$
To prove this upper bound in general, suppose to the contrary that for the resulting decision trees $T_O$ and $T_M$ we have $$\sum_{i=1}^n R(x_i, y_i, T_O(g(x_i)) - \sum_{i=1}^n R(x_i, y_i, T_M(g(x_i)) > \frac{n(m-1)}{m}R_{max}.$$ 
Since $R_{max}$ is an upper bound on the difference in reward among all models for the given dataset, we know, for all $i \in [n]$,
$$R(x_i, y_i, T_O(g(x_i)) - R(x_i, y_i, T_M(g(x_i)) \leq R_{max}.$$
It follows that for at least $\frac{n(m-1)}{m}$ samples, $\max_{j \in [m]}R(x_i, y_i, h_j) > R(x_i, y_i, T_M(x_i))$. Since, in the Meta-Tree setup, the samples are labeled according to the constituent model achieving the maximum reward, this implies that $T_M$ misclassifies more than $\frac{n(m-1)}{m}$ samples. However, there must be at least $\frac{n}{m}$ samples in the majority class, such that predicting the majority class alone achieves a lower misclassification error than $T_M$. This leads to a contradiction, as a depth-zero decision tree would achieve a lower misclassification error than $T_M$. This gives us our result and provides an intuitive upper bound on the potential gain of using relative rewards as a function of the maximum difference in rewards between constituent models. The same argument holds if we replace the Meta-Tree approach with Meta-XGB. In that case, we can consider a corresponding subset of the data $A \subseteq \{(x_i, q_i)\}_{i=1}^n$ that is not separable by a boosted tree model.
\end{proof}
\subsection{Proof of Proposition \ref{prop:rejection-converge}}\label{appx-rejection-properties}
\begin{proof}
    We begin with the first property. Suppose $P_D(R(x,f^*(x), h_i) = R_{ub}) = 0$. Then for any connected subspace $A \subseteq \mathcal{X}$ with positive measure, $P_D(A) > 0$, we know 
    $$E_{\mathcal{D}}[R(x, f^*(x), h_i) \cdot \mathbbm{1}\{x \in A\}] < R_{ub} \quad \forall i \in [m].$$
    It then follows immediately that since $\alpha_n \rightarrow R_{ub}$, there exists some $N$ such that for all $n \geq N$, 
    $$E_{\mathcal{D}}[R(x, f^*(x), h_r^n) \cdot \mathbbm{1}\{x \in A\}] > \max_{h \in H}E_{\mathcal{D}}[R(x, f^*(x), h) \cdot \mathbbm{1}\{x \in A\}].$$
    Since this reasoning holds for any maximal dominated subspace, we have our result.

    The first part of the second statement holds by definition since we state $s_j^{n+1} \subseteq s_i^n$. The second part follows from the fact that $s_j^{n+1} \subseteq s_i^n$ and the expected reward for the corresponding dominant model $h^*_{s_j^{n+1}}$ must be greater than $\alpha_{n+1}$, which is greater than $\alpha_n$. Since we can always set $h^*_{s_j^{n+1}} = h^*_{s_j^{n}}$, the expected reward for $h^*_{s_j^{n+1}}$ over $s_j^{n+1}$ must always be at least the same as the expected reward for $h^*_{s_j^{n}}$. Therefore, the inequality holds.
\end{proof}

\section{Full Concrete and Projectile Motion OP\textsuperscript{2}Ts}
Below, we reproduce the full-depth OP\textsuperscript{2}Ts that achieve the best out-of-sample MSE for both the concrete compressive strength and projectile motion datasets from our experiments in Section \ref{sec:concrete} and Section \ref{sec:projectile}, respectively.

\begin{figure}[h!]
    \centering
    \includegraphics[width=1\linewidth]{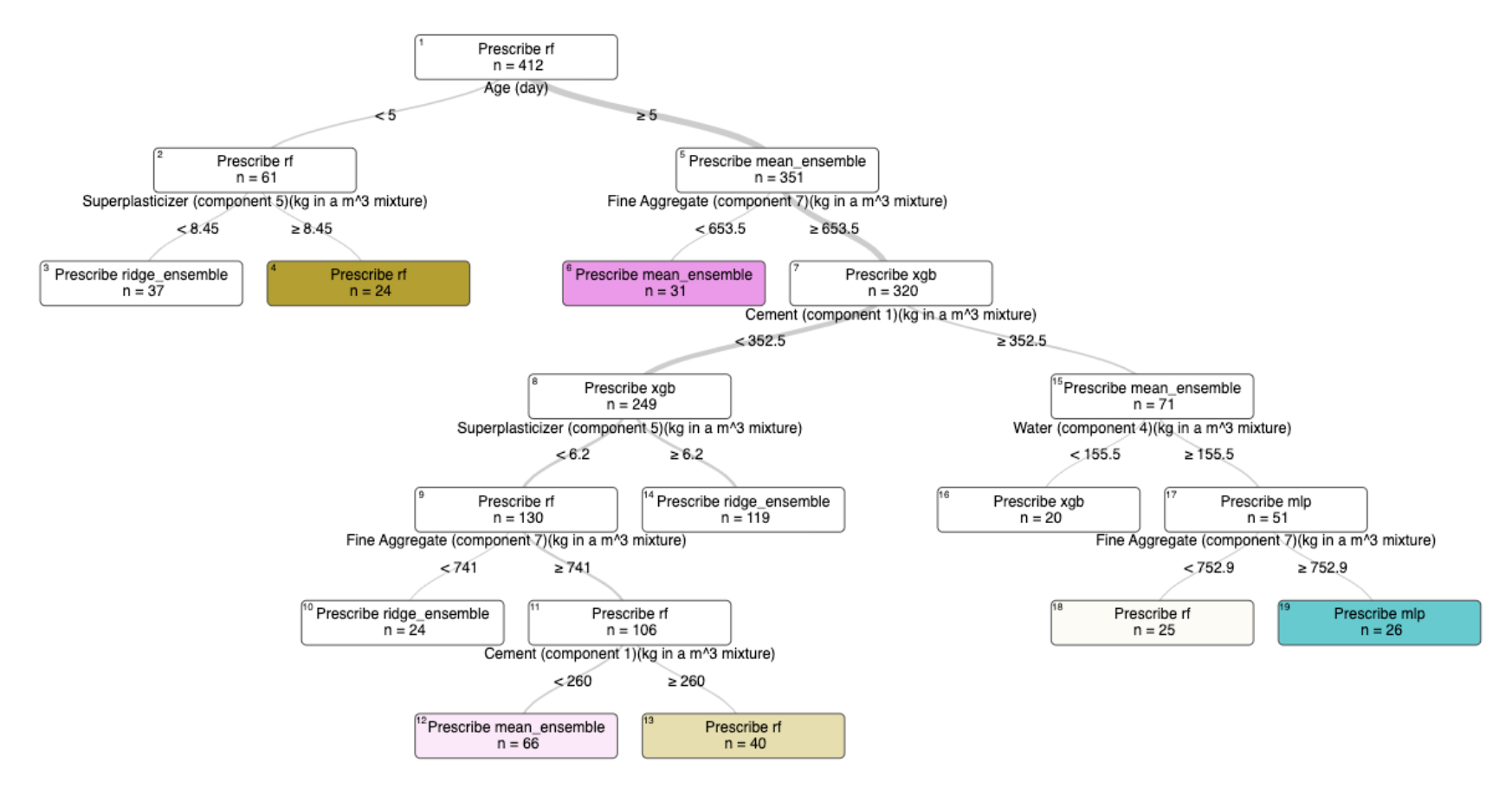}
    \caption{The OP\textsuperscript{2}T with max depth $d=10$ fit on the Concrete Compressive Strength dataset, achieving a 7.8\% reduction in MSE over the best constituent model in hindsight. Depending on the context, the OP\textsuperscript{2}T adaptively selects different models and model ensembles.}
    \label{fig:concrete-full-depth-viz}
\end{figure}

\begin{sidewaysfigure}[h!]
    \centering
    \includegraphics[width=1\linewidth]{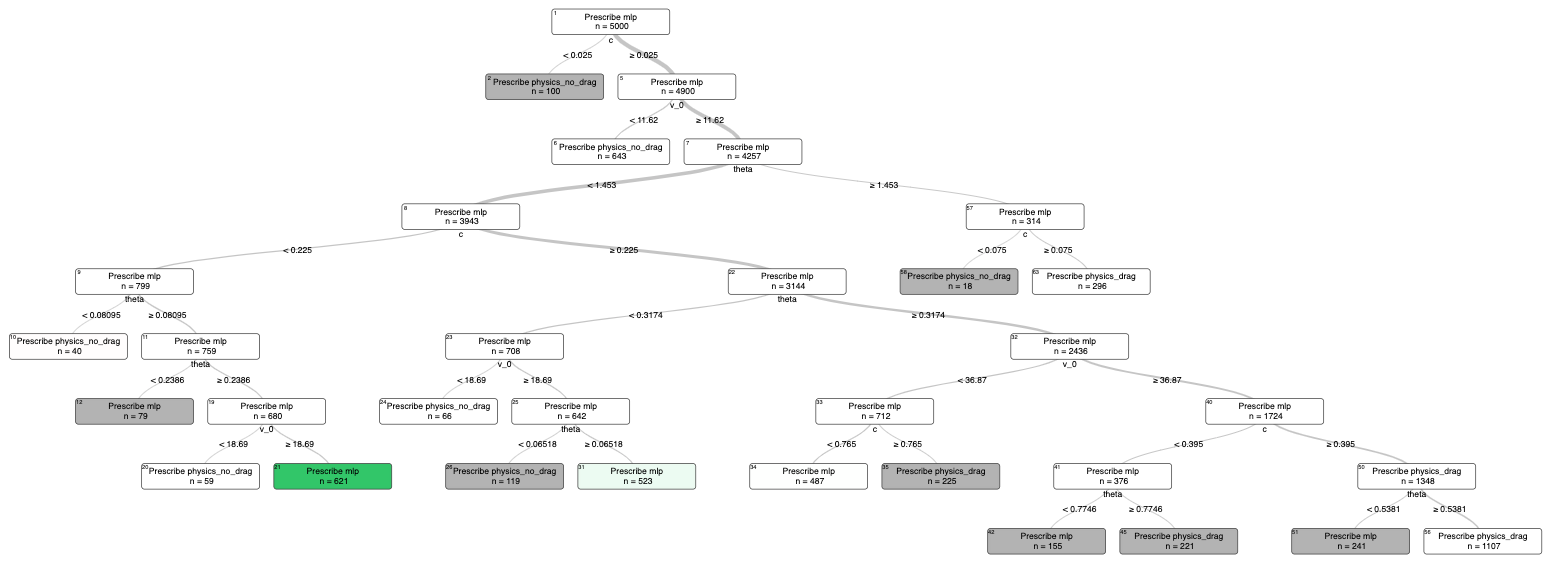}
    \caption{The OP\textsuperscript{2}T model, without rejection, fit on projectile motion data with both physics-based and ML constituent models. The prescription ``mlp'' corresponds to the MLP model, ``physics\_no\_drag'' corresponds to Eq.  \eqref{eq:physics-no-drag}, and ``phyics\_drag'' corresponds to Eq. \eqref{eq:physics-drag}.}
    \label{fig:physics-no-reject}
\end{sidewaysfigure}

\clearpage
\vskip 0.2in
\bibliography{main}

\begin{thebibliography}{36}
\providecommand{\natexlab}[1]{#1}
\providecommand{\url}[1]{\texttt{#1}}
\expandafter\ifx\csname urlstyle\endcsname\relax
  \providecommand{\doi}[1]{doi: #1}\else
  \providecommand{\doi}{doi: \begingroup \urlstyle{rm}\Url}\fi

\bibitem[Amram et~al.(2022)Amram, Dunn, and Zhuo]{amram_optimal_2022}
Maxime Amram, Jack Dunn, and Ying~Daisy Zhuo.
\newblock Optimal policy trees.
\newblock \emph{Machine Learning}, 111\penalty0 (7):\penalty0 2741--2768, July
  2022.
\newblock ISSN 1573-0565.
\newblock \doi{10.1007/s10994-022-06128-5}.
\newblock URL \url{https://doi.org/10.1007/s10994-022-06128-5}.

\bibitem[Bartlett and Wegkamp(2008)]{bartlett_classification_2008}
Peter~L. Bartlett and Marten~H. Wegkamp.
\newblock Classification with a {Reject} {Option} using a {Hinge} {Loss}.
\newblock \emph{Journal of Machine Learning Research}, 9\penalty0
  (59):\penalty0 1823--1840, 2008.
\newblock ISSN 1533-7928.
\newblock URL \url{http://jmlr.org/papers/v9/bartlett08a.html}.

\bibitem[Bertsimas and Dunn(2017)]{bertsimas_optimal_2017}
Dimitris Bertsimas and Jack Dunn.
\newblock Optimal classification trees.
\newblock \emph{Machine Learning}, 106\penalty0 (7):\penalty0 1039--1082, July
  2017.
\newblock ISSN 0885-6125, 1573-0565.
\newblock \doi{10.1007/s10994-017-5633-9}.
\newblock URL \url{http://link.springer.com/10.1007/s10994-017-5633-9}.

\bibitem[Bertsimas and Dunn(2019)]{bertsimas_machine_2019}
Dimitris Bertsimas and Jack Dunn.
\newblock \emph{Machine learning under a modern optimization lens}.
\newblock Dynamic Ideas LLC Charlestown, MA, 2019.
\newblock URL
  \url{https://www.lnmb.nl/conferences/2018/programlnmbconference/Bertsimas-1.pdf}.

\bibitem[Biggs et~al.(2021)Biggs, Sun, and Ettl]{biggs_model_2021}
Max Biggs, Wei Sun, and Markus Ettl.
\newblock Model {Distillation} for {Revenue} {Optimization}: {Interpretable}
  {Personalized} {Pricing}, June 2021.
\newblock URL \url{http://arxiv.org/abs/2007.01903}.
\newblock arXiv:2007.01903 [cs, stat].

\bibitem[Boussioux et~al.(2022)Boussioux, Zeng, Guénais, and
  Bertsimas]{boussioux_hurricane_2022}
Léonard Boussioux, Cynthia Zeng, Théo Guénais, and Dimitris Bertsimas.
\newblock Hurricane {Forecasting}: {A} {Novel} {Multimodal} {Machine}
  {Learning} {Framework}.
\newblock \emph{Weather and Forecasting}, 37\penalty0 (6):\penalty0 817--831,
  June 2022.
\newblock ISSN 1520-0434, 0882-8156.
\newblock \doi{10.1175/WAF-D-21-0091.1}.
\newblock URL
  \url{https://journals.ametsoc.org/view/journals/wefo/37/6/WAF-D-21-0091.1.xml}.
\newblock Publisher: American Meteorological Society Section: Weather and
  Forecasting.

\bibitem[Breiman et~al.(1984)Breiman, Friedman, Stone, and
  Olshen]{breiman_classification_1984}
L.~Breiman, J.~Friedman, C.J. Stone, and R.A. Olshen.
\newblock \emph{Classification and {Regression} {Trees}}.
\newblock Taylor \& Francis, 1984.
\newblock ISBN 978-0-412-04841-8.
\newblock URL \url{https://books.google.com/books?id=JwQx-WOmSyQC}.

\bibitem[Charoenphakdee et~al.(2021)Charoenphakdee, Cui, Zhang, and
  Sugiyama]{charoenphakdee_classification_2021}
Nontawat Charoenphakdee, Zhenghang Cui, Yivan Zhang, and Masashi Sugiyama.
\newblock Classification with {Rejection} {Based} on {Cost}-sensitive
  {Classification}, September 2021.
\newblock URL \url{http://arxiv.org/abs/2010.11748}.
\newblock arXiv:2010.11748 [cs, stat].

\bibitem[Chen and Guestrin(2016)]{xgboost}
Tianqi Chen and Carlos Guestrin.
\newblock {XGBoost}.
\newblock In \emph{Proceedings of the 22nd {ACM} {SIGKDD} International
  Conference on Knowledge Discovery and Data Mining}. {ACM}, aug 2016.
\newblock \doi{10.1145/2939672.2939785}.
\newblock URL \url{https://doi.org/10.1145%2F2939672.2939785}.

\bibitem[Chow(1970)]{chow_optimum_1970}
C.~Chow.
\newblock On optimum recognition error and reject tradeoff.
\newblock \emph{IEEE Transactions on Information Theory}, 16\penalty0
  (1):\penalty0 41--46, January 1970.
\newblock ISSN 1557-9654.
\newblock \doi{10.1109/TIT.1970.1054406}.
\newblock URL \url{https://ieeexplore.ieee.org/document/1054406}.
\newblock Conference Name: IEEE Transactions on Information Theory.

\bibitem[Cortes et~al.(2016)Cortes, DeSalvo, and Mohri]{ortner_learning_2016}
Corinna Cortes, Giulia DeSalvo, and Mehryar Mohri.
\newblock Learning with {Rejection}.
\newblock In Ronald Ortner, Hans~Ulrich Simon, and Sandra Zilles, editors,
  \emph{Algorithmic {Learning} {Theory}}, volume 9925, pages 67--82. Springer
  International Publishing, Cham, 2016.
\newblock ISBN 978-3-319-46378-0 978-3-319-46379-7.
\newblock \doi{10.1007/978-3-319-46379-7_5}.
\newblock URL \url{http://link.springer.com/10.1007/978-3-319-46379-7_5}.
\newblock Series Title: Lecture Notes in Computer Science.

\bibitem[Devlin et~al.(2019)Devlin, Chang, Lee, and
  Toutanova]{devlin_bert_2019}
Jacob Devlin, Ming-Wei Chang, Kenton Lee, and Kristina Toutanova.
\newblock {BERT}: {Pre}-training of {Deep} {Bidirectional} {Transformers} for
  {Language} {Understanding}, May 2019.
\newblock URL \url{http://arxiv.org/abs/1810.04805}.
\newblock arXiv:1810.04805 [cs].

\bibitem[Ding et~al.(2018)Ding, Tarokh, and Yang]{ding2018model}
Jie Ding, Vahid Tarokh, and Yuhong Yang.
\newblock Model selection techniques: An overview.
\newblock \emph{IEEE Signal Processing Magazine}, 35\penalty0 (6):\penalty0
  16--34, 2018.

\bibitem[Dong et~al.(2020)Dong, Yu, Cao, Shi, and Ma]{dong_survey_2020}
Xibin Dong, Zhiwen Yu, Wenming Cao, Yifan Shi, and Qianli Ma.
\newblock A survey on ensemble learning.
\newblock \emph{Frontiers of Computer Science}, 14\penalty0 (2):\penalty0
  241--258, April 2020.
\newblock ISSN 2095-2236.
\newblock \doi{10.1007/s11704-019-8208-z}.
\newblock URL \url{https://doi.org/10.1007/s11704-019-8208-z}.

\bibitem[Hendrickx et~al.(2021)Hendrickx, Perini, Van~der Plas, Meert, and
  Davis]{hendrickx2021machine}
Kilian Hendrickx, Lorenzo Perini, Dries Van~der Plas, Wannes Meert, and Jesse
  Davis.
\newblock Machine learning with a reject option: A survey.
\newblock \emph{arXiv preprint arXiv:2107.11277}, 2021.

\bibitem[Ismail et~al.(2023)Ismail, Arik, Yoon, Taly, Feizi, and
  Pfister]{ismail_interpretable_2023}
Aya~Abdelsalam Ismail, Sercan~O. Arik, Jinsung Yoon, Ankur Taly, Soheil Feizi,
  and Tomas Pfister.
\newblock Interpretable {Mixture} of {Experts}.
\newblock \emph{Transactions on Machine Learning Research}, March 2023.
\newblock ISSN 2835-8856.
\newblock URL \url{https://openreview.net/forum?id=DdZoPUPm0a}.

\bibitem[Johnson et~al.(2023)Johnson, Bulgarelli, Shen, Gayles, Shammout,
  Horng, Pollard, Hao, Moody, Gow, Lehman, Celi, and
  Mark]{johnson_mimic-iv_2023}
Alistair E.~W. Johnson, Lucas Bulgarelli, Lu~Shen, Alvin Gayles, Ayad Shammout,
  Steven Horng, Tom~J. Pollard, Sicheng Hao, Benjamin Moody, Brian Gow,
  Li-wei~H. Lehman, Leo~A. Celi, and Roger~G. Mark.
\newblock {MIMIC}-{IV}, a freely accessible electronic health record dataset.
\newblock \emph{Scientific Data}, 10\penalty0 (1):\penalty0 1, January 2023.
\newblock ISSN 2052-4463.
\newblock \doi{10.1038/s41597-022-01899-x}.
\newblock URL \url{https://www.nature.com/articles/s41597-022-01899-x}.
\newblock Number: 1 Publisher: Nature Publishing Group.

\bibitem[Jordan and Jacobs(1994)]{jordan_hierarchical_1994}
Michael~I. Jordan and Robert~A. Jacobs.
\newblock Hierarchical {Mixtures} of {Experts} and the {EM} {Algorithm}.
\newblock \emph{Neural Computation}, 6\penalty0 (2):\penalty0 181--214, March
  1994.
\newblock ISSN 0899-7667.
\newblock \doi{10.1162/neco.1994.6.2.181}.
\newblock URL \url{https://ieeexplore.ieee.org/abstract/document/6796382}.
\newblock Conference Name: Neural Computation.

\bibitem[Maas et~al.(2011)Maas, Daly, Pham, Huang, Ng, and
  Potts]{imdb_original}
Andrew~L. Maas, Raymond~E. Daly, Peter~T. Pham, Dan Huang, Andrew~Y. Ng, and
  Christopher Potts.
\newblock Learning word vectors for sentiment analysis.
\newblock In \emph{Proceedings of the 49th Annual Meeting of the Association
  for Computational Linguistics: Human Language Technologies}, pages 142--150,
  Portland, Oregon, USA, June 2011. Association for Computational Linguistics.
\newblock URL \url{http://www.aclweb.org/anthology/P11-1015}.

\bibitem[Masoudnia and Ebrahimpour(2014)]{masoudnia_mixture_2014}
Saeed Masoudnia and Reza Ebrahimpour.
\newblock Mixture of experts: a literature survey.
\newblock \emph{Artificial Intelligence Review}, 42\penalty0 (2):\penalty0
  275--293, August 2014.
\newblock ISSN 1573-7462.
\newblock \doi{10.1007/s10462-012-9338-y}.
\newblock URL \url{https://doi.org/10.1007/s10462-012-9338-y}.

\bibitem[Nori et~al.(2019)Nori, Jenkins, Koch, and
  Caruana]{nori2019interpretml}
Harsha Nori, Samuel Jenkins, Paul Koch, and Rich Caruana.
\newblock Interpretml: A unified framework for machine learning
  interpretability.
\newblock \emph{arXiv preprint arXiv:1909.09223}, 2019.

\bibitem[Organization(2023)]{WHO}
World~Health Organization.
\newblock {The Global Health Observatory}, 2023.
\newblock URL
  \url{https://www.who.int/data/gho/indicator-metadata-registry/imr-details/2380#:~:text=The%20expected%20values%20for%20normal,and%20monitoring%20glycemia%20are%20recommended.}

\bibitem[Paszke et~al.(2019)Paszke, Gross, Massa, Lerer, Bradbury, Chanan,
  Killeen, Lin, Gimelshein, Antiga, Desmaison, Köpf, Yang, DeVito, Raison,
  Tejani, Chilamkurthy, Steiner, Fang, Bai, and Chintala]{paszke2019pytorch}
Adam Paszke, Sam Gross, Francisco Massa, Adam Lerer, James Bradbury, Gregory
  Chanan, Trevor Killeen, Zeming Lin, Natalia Gimelshein, Luca Antiga, Alban
  Desmaison, Andreas Köpf, Edward Yang, Zach DeVito, Martin Raison, Alykhan
  Tejani, Sasank Chilamkurthy, Benoit Steiner, Lu~Fang, Junjie Bai, and Soumith
  Chintala.
\newblock Pytorch: An imperative style, high-performance deep learning library,
  2019.

\bibitem[Pedregosa et~al.(2011)Pedregosa, Varoquaux, Gramfort, Michel, Thirion,
  Grisel, Blondel, Prettenhofer, Weiss, Dubourg, Vanderplas, Passos,
  Cournapeau, Brucher, Perrot, and Duchesnay]{scikit-learn}
F.~Pedregosa, G.~Varoquaux, A.~Gramfort, V.~Michel, B.~Thirion, O.~Grisel,
  M.~Blondel, P.~Prettenhofer, R.~Weiss, V.~Dubourg, J.~Vanderplas, A.~Passos,
  D.~Cournapeau, M.~Brucher, M.~Perrot, and E.~Duchesnay.
\newblock Scikit-learn: Machine learning in {P}ython.
\newblock \emph{Journal of Machine Learning Research}, 12:\penalty0 2825--2830,
  2011.

\bibitem[Provost and Fawcett(2001)]{provost_robust_2001}
Foster Provost and Tom Fawcett.
\newblock Robust {Classification} for {Imprecise} {Environments}.
\newblock \emph{Machine Learning}, 42\penalty0 (3):\penalty0 203--231, March
  2001.
\newblock ISSN 1573-0565.
\newblock \doi{10.1023/A:1007601015854}.
\newblock URL \url{https://doi.org/10.1023/A:1007601015854}.

\bibitem[Riquelme et~al.(2021)Riquelme, Puigcerver, Mustafa, Neumann, Jenatton,
  Susano~Pinto, Keysers, and Houlsby]{riquelme2021scaling}
Carlos Riquelme, Joan Puigcerver, Basil Mustafa, Maxim Neumann, Rodolphe
  Jenatton, Andr{\'e} Susano~Pinto, Daniel Keysers, and Neil Houlsby.
\newblock Scaling vision with sparse mixture of experts.
\newblock \emph{Advances in Neural Information Processing Systems},
  34:\penalty0 8583--8595, 2021.

\bibitem[Rudin et~al.(2021)Rudin, Chen, Chen, Huang, Semenova, and
  Zhong]{rudin_interpretable_2021}
Cynthia Rudin, Chaofan Chen, Zhi Chen, Haiyang Huang, Lesia Semenova, and Chudi
  Zhong.
\newblock Interpretable {Machine} {Learning}: {Fundamental} {Principles} and 10
  {Grand} {Challenges}, July 2021.
\newblock URL \url{http://arxiv.org/abs/2103.11251}.
\newblock arXiv:2103.11251 [cs, stat].

\bibitem[Shazeer et~al.(2017)Shazeer, Mirhoseini, Maziarz, Davis, Le, Hinton,
  and Dean]{shazeer2017outrageously}
Noam Shazeer, Azalia Mirhoseini, Krzysztof Maziarz, Andy Davis, Quoc Le,
  Geoffrey Hinton, and Jeff Dean.
\newblock Outrageously large neural networks: The sparsely-gated
  mixture-of-experts layer.
\newblock \emph{arXiv preprint arXiv:1701.06538}, 2017.

\bibitem[Singh et~al.(2023)Singh, Askari, Caruana, and
  Gao]{singh_augmenting_2023}
Chandan Singh, Armin Askari, Rich Caruana, and Jianfeng Gao.
\newblock Augmenting interpretable models with large language models during
  training.
\newblock \emph{Nature Communications}, 14\penalty0 (1):\penalty0 7913,
  November 2023.
\newblock ISSN 2041-1723.
\newblock \doi{10.1038/s41467-023-43713-1}.
\newblock URL \url{https://www.nature.com/articles/s41467-023-43713-1}.
\newblock Number: 1 Publisher: Nature Publishing Group.

\bibitem[Soenksen et~al.(2022)Soenksen, Ma, Zeng, Boussioux,
  Villalobos~Carballo, Na, Wiberg, Li, Fuentes, and
  Bertsimas]{soenksen_integrated_2022}
Luis~R. Soenksen, Yu~Ma, Cynthia Zeng, Leonard Boussioux, Kimberly
  Villalobos~Carballo, Liangyuan Na, Holly~M. Wiberg, Michael~L. Li, Ignacio
  Fuentes, and Dimitris Bertsimas.
\newblock Integrated multimodal artificial intelligence framework for
  healthcare applications.
\newblock \emph{npj Digital Medicine}, 5\penalty0 (1):\penalty0 1--10,
  September 2022.
\newblock ISSN 2398-6352.
\newblock \doi{10.1038/s41746-022-00689-4}.
\newblock URL \url{https://www.nature.com/articles/s41746-022-00689-4}.
\newblock Number: 1 Publisher: Nature Publishing Group.

\bibitem[Solomatine and Siek(2004)]{solomatine_semi-optimal_2004}
D.P. Solomatine and M.B.L.A. Siek.
\newblock Semi-optimal hierarchical regression models and {ANNs}.
\newblock In \emph{2004 {IEEE} {International} {Joint} {Conference} on {Neural}
  {Networks} ({IEEE} {Cat}. {No}.{04CH37541})}, volume~2, pages 1173--1177
  vol.2, July 2004.
\newblock \doi{10.1109/IJCNN.2004.1380104}.
\newblock URL \url{https://ieeexplore.ieee.org/document/1380104}.
\newblock ISSN: 1098-7576.

\bibitem[Wang et~al.(2021)Wang, Feng, and Zhang]{wang_rethinking_2021}
Deng-Bao Wang, Lei Feng, and Min-Ling Zhang.
\newblock Rethinking {Calibration} of {Deep} {Neural} {Networks}: {Do} {Not}
  {Be} {Afraid} of {Overconfidence}.
\newblock In \emph{Advances in {Neural} {Information} {Processing} {Systems}},
  volume~34, pages 11809--11820. Curran Associates, Inc., 2021.
\newblock URL
  \url{https://proceedings.neurips.cc/paper/2021/hash/61f3a6dbc9120ea78ef75544826c814e-Abstract.html}.

\bibitem[Wolf et~al.(2020)Wolf, Debut, Sanh, Chaumond, Delangue, Moi, Cistac,
  Rault, Louf, Funtowicz, Davison, Shleifer, von Platen, Ma, Jernite, Plu, Xu,
  Scao, Gugger, Drame, Lhoest, and Rush]{wolf2020huggingfaces}
Thomas Wolf, Lysandre Debut, Victor Sanh, Julien Chaumond, Clement Delangue,
  Anthony Moi, Pierric Cistac, Tim Rault, Rémi Louf, Morgan Funtowicz, Joe
  Davison, Sam Shleifer, Patrick von Platen, Clara Ma, Yacine Jernite, Julien
  Plu, Canwen Xu, Teven~Le Scao, Sylvain Gugger, Mariama Drame, Quentin Lhoest,
  and Alexander~M. Rush.
\newblock Huggingface's transformers: State-of-the-art natural language
  processing, 2020.

\bibitem[Yeh(2007)]{misc_concrete_compressive_strength_165}
I-Cheng Yeh.
\newblock {Concrete Compressive Strength}.
\newblock UCI Machine Learning Repository, 2007.
\newblock {DOI}: https://doi.org/10.24432/C5PK67.

\bibitem[Zhou et~al.(2018)Zhou, Athey, and Wager]{zhou_offline_2018}
Zhengyuan Zhou, Susan Athey, and Stefan Wager.
\newblock Offline {Multi}-{Action} {Policy} {Learning}: {Generalization} and
  {Optimization}, November 2018.
\newblock URL \url{http://arxiv.org/abs/1810.04778}.
\newblock arXiv:1810.04778 [cs, econ, stat].

\bibitem[İrsoy et~al.(2012)İrsoy, Yıldız, and Alpaydın]{irsoy_soft_2012}
Ozan İrsoy, Olcay~Taner Yıldız, and Ethem Alpaydın.
\newblock Soft decision trees.
\newblock In \emph{Proceedings of the 21st {International} {Conference} on
  {Pattern} {Recognition} ({ICPR2012})}, pages 1819--1822, November 2012.
\newblock URL
  \url{https://ieeexplore.ieee.org/abstract/document/6460506?casa_token=6bhp6-2UopUAAAAA:O140c1pKYGBDgQYLQkoZxDmZLBWy3d9jvHA3GrGNfmgYWD-Bg_Wk3r58XloL008VQK4dvgN6}.
\newblock ISSN: 1051-4651.

\end{thebibliography}

\end{document}